\documentclass[letterpaper]{article} 
\usepackage[preprint]{aaai2027}  
\usepackage[hyphens]{url}  
\usepackage{graphicx} 
\usepackage{natbib}  
\usepackage{caption} 
\usepackage{amsmath,amssymb,amsthm}
\theoremstyle{definition}

\theoremstyle{plain}
\newtheorem{theorem}{Theorem}
\usepackage{algorithm}
\usepackage{algorithmic}

\usepackage{newfloat}
\usepackage{amsfonts}
\usepackage{listings}
\DeclareCaptionStyle{ruled}{labelfont=normalfont,labelsep=colon,strut=off} 
\lstdefinestyle{cadexample}{%
    basicstyle={\scriptsize\ttfamily},
    columns=fullflexible,
    keepspaces=true,
    showstringspaces=false,
    tabsize=2,
    breaklines=true,
    numbers=none,
    frame=single,
    framerule=0.3pt,
    rulecolor=\color[gray]{0.72},
    backgroundcolor=\color[gray]{0.97},
    xleftmargin=0.4em,
    xrightmargin=0.4em,
    framexleftmargin=0.5em,
    framexrightmargin=0.5em,
    aboveskip=2pt,
    belowskip=2pt
}
\floatstyle{plain}
\newfloat{listing}{tb}{lst}{}
\floatname{listing}{Listing}

\usepackage{booktabs}
\usepackage{array}
\usepackage{fontawesome5}
\usepackage[colorlinks=true,breaklinks=true]{hyperref}
\usepackage{cleveref}
\usepackage{multirow}
\usepackage{subcaption}
\title{\textsc{DeciCAD}: Harness-Grounded Learning of Action Policies for CAD Agents}

\author{
    Yufeng~Peng\textsuperscript{\rm 1},
    Yixin~Zhang\textsuperscript{\rm 1},
    Ziyi~Liu\textsuperscript{\rm 1},
    Zhikun~Xu\textsuperscript{\rm 1},
    Jiarong~Yang\textsuperscript{\rm 1},\\
    Sijie~Wang\textsuperscript{\rm 2},
    Wenjie~Feng\textsuperscript{\rm 1},
    Qiyu Kang\textsuperscript{\rm 1}\setcounter{footnote}{1}\thanks{
        Correspondence to: Qiyu Kang $<$qiyukang@ustc.edu.cn$>$.}
}
\affiliations{
    \textsuperscript{\rm 1}University of Science and Technology of China\\
    \textsuperscript{\rm 2}Nanyang Technological University\\
    \vspace{3pt}
    {\normalsize\color{LinkBlue}\faGithub\hspace{0.4em}%
    \href{https://github.com/PhysAGI/PhysCAD}{\texttt{\textbf{github.com/PhysAGI/PhysCAD}}}}
}

\definecolor{DeciBlue}{HTML}{315B73}
\definecolor{DeciBlueSoft}{HTML}{F1F6F8}
\definecolor{DeciCodeText}{HTML}{24343D}
\definecolor{DeciTeal}{HTML}{287A78}
\definecolor{DeciPlum}{HTML}{84536F}
\lstdefinestyle{suppcode}{
    basicstyle=\fontsize{7.2}{8.4}\selectfont\ttfamily\color{DeciCodeText},
    columns=fullflexible,
    keepspaces=true,
    showstringspaces=false,
    breaklines=true,
    frame=none,
    backgroundcolor=\color{DeciBlueSoft},
    framesep=0pt,
    framexleftmargin=6pt,
    framexrightmargin=6pt,
    xleftmargin=6pt,
    xrightmargin=6pt,
    aboveskip=5pt,
    belowskip=5pt,
    keywordstyle=[1]\bfseries\color{DeciBlue},
    morekeywords=[1]{CADAction,Document,Sketch,Solid,Entity,Assembly,Inspect,
        Extrude,SpatialContract,CreateSketch,SetPlacement,Execute,
        ActionResult,InspectList,InspectGet},
    keywordstyle=[2]\color{DeciTeal},
    morekeywords=[2]{Ref,StableRef,Int,Float,Point2D,Point3D,Placement,
        Frame,Generation,Optional,ArgumentPath,ArtifactPath,
        ActionSpecificEvidence,ShortExplanation},
    keywordstyle=[3]\bfseries\color{DeciPlum},
    morekeywords=[3]{Create,Open,Save,Export,AddGeometry,Constrain,Pattern,
        Trim,Project,CreateBody,Revolve,Sweep,Loft,Boolean,Repeat,Fillet,
        Chamfer,Shell,AddComponent,Ground,AddMate,RemoveMate,
        Solve,List,Get,Add,Cut,Distance,ThroughAll,SketchNormal,Vector3D,
        World,Parent,Success,Failure,InvalidInput,MissingObject,InvalidState,
        GeometryFailure,Conflict,Timeout,BackendUnavailable,OutcomeUnknown,
        RetryWithArguments,InspectState,ChooseAnotherAction,if,else,return}
}

\definecolor{LinkBlue}{rgb}{0.0,0.25,0.65}
\hypersetup{
    linkcolor=black,
    citecolor=black,
    urlcolor=LinkBlue,
    pdftitle={DeciCAD: Harness-Grounded Learning of Action Policies for CAD Agents},
    pdfauthor={Yufeng Peng, Yixin Zhang, Ziyi Liu, Zhikun Xu, Jiarong Yang, Sijie Wang, Wenjie Feng, Qiyu Kang}
}

\begin{document}

\maketitle

\begin{abstract}
    Large Language Model agents operate through closed-loop action policies whose decision space is defined by an agent harness.
    We investigate how this operating space shapes agent behavior and policy learning in computer-aided design (CAD), a professional domain characterized by persistent state, geometric constraints, and long-horizon procedures.
    We instantiate this perspective in \textsc{DeciCAD}, a domain-native harness that exposes semantic CAD actions, decision-oriented observations, environment feedback, and persistent state transitions.
    Building on this harness, we develop a learning pipeline that extracts causal decision supervision from interaction trajectories and improves the policy through supervised fine-tuning (SFT) and geometry-grounded online reinforcement learning (RL).
    Across four CAD benchmarks, \textsc{DeciCAD} achieves the best F1 and intersection over union (IoU) while reducing model-output tokens by 45--88\% relative to a general-purpose coding harness.
    Under matched supervision, its SFT policy improves F1 from 0.6883 to 0.7437 and outperforms the coding policy; subsequent online RL further raises geometry availability from 66.5\% to 73.5\% and F1 to 0.7880.
    These results demonstrate that the decision space induced by a harness materially affects both agent execution and policy learning in CAD.
\end{abstract}

\section{Introduction}

Large Language Models (LLMs) are increasingly being used as agents to operate
professional software and complete long-horizon
tasks~\citep{achiam2023gpt,naveed2025comprehensive,guo2024large}.
Computer-aided design (CAD) is a particularly demanding setting for
such agents~\citep{willis2020fusion360gallery,wu2021deepcad}.
Unlike one-shot content generation, CAD modeling requires a sequence of
state-dependent operations over a persistent design document.
An agent must create and modify sketches, constraints, features, and geometric objects;
inspect intermediate results; recover from execution failures;
and maintain references across long interaction trajectories.
In this process, the model outputs are interpreted as actions:
each action changes the current CAD state, produces execution feedback, and conditions
the next decision~\citep{yao2022react,schick2023toolformer}.
The effectiveness of a CAD agent therefore depends not only on the underlying model
and its reasoning strategy,
but also on how CAD actions, states, and execution outcomes
are represented to the model.

\begin{figure}[t]
    \centering
    \includegraphics[width=\linewidth]{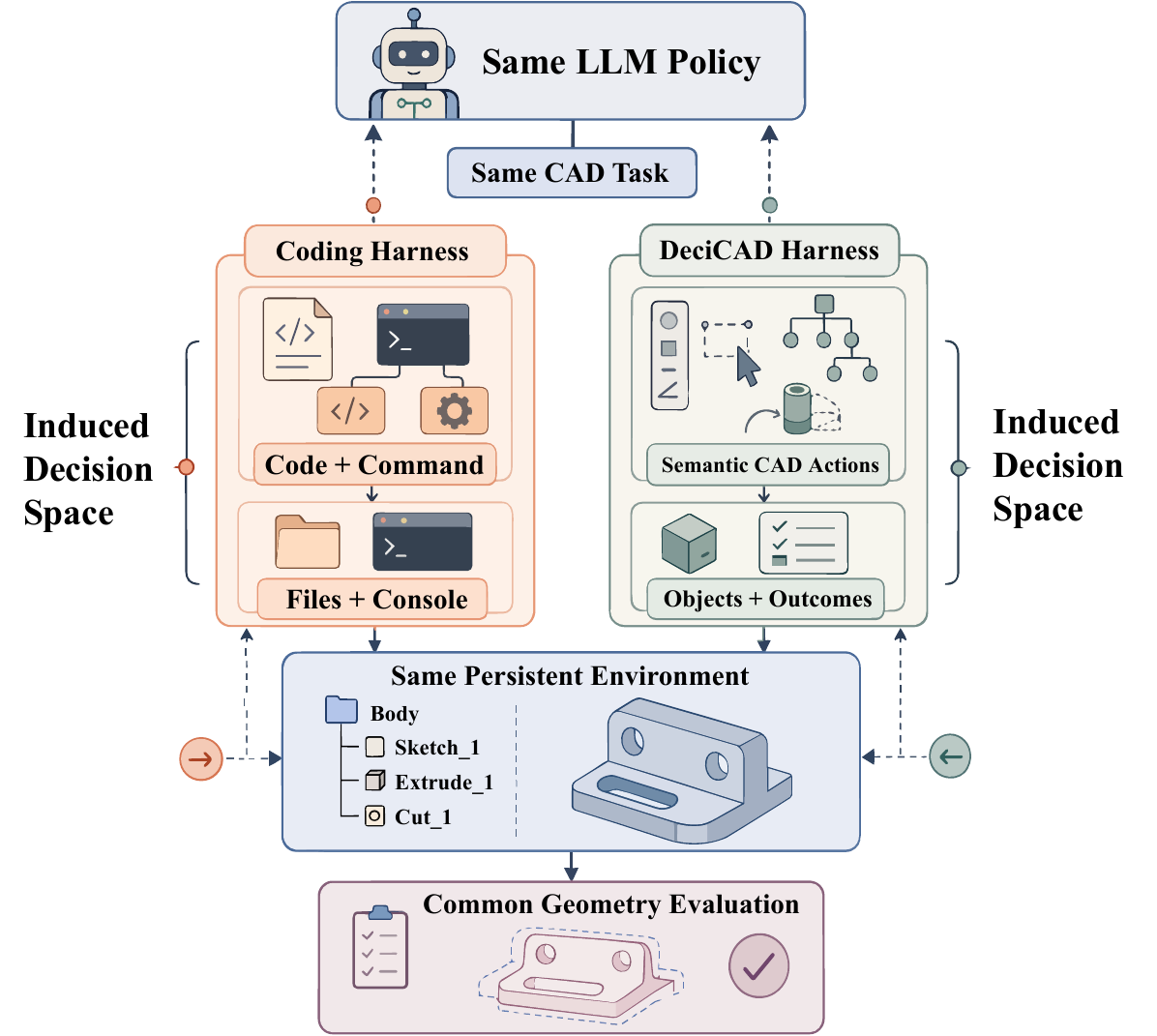}
    \caption{Harness-induced CAD decision spaces under controlled model and
        environment conditions. Given the same LLM policy, CAD task,
        persistent design state, and geometry evaluator, the coding harness
        represents decisions through code and shell commands and returns files
        and console output. \textsc{DeciCAD} instead exposes semantic CAD
        actions and returns object-level state and execution outcomes. The two
        interfaces therefore induce different action--observation histories
        over the same underlying modeling process.}
    \label{fig:harness-induced-decision-space}
\end{figure}

Such a model--environment boundary is defined by an \emph{agent harness}
\citep{yang2024sweagent}. In CAD, a harness determines which design operations the
model can express, which parts of the evolving design state it can inspect, and
how execution outcomes are returned for subsequent decisions.
A general-purpose coding harness may expose the CAD system through programs,
libraries, and runtime feedback, whereas a domain-native harness may expose
structured CAD actions and object-level observations.
Consequently, even when the model, task, and CAD backend remain fixed,
different harnesses can present the model with substantially different
decision spaces, as illustrated in
Figure~\ref{fig:harness-induced-decision-space}.
Existing work has studied CAD code generation, tool-using agents, execution repair,
and interactive modeling~\citep{mallis2024cadassistant,guan2025cadcoder,doris2025cadcoder,
    wang2025cadfusion,barkley2026cadsmith,
    gong-etal-2026-toolcad},
but the harness itself is typically treated as implementation infrastructure
rather than as a variable that shapes both agent execution and policy learning.
This raises a central question:
\emph{How should a CAD agent harness represent actions and states so that it
    preserves the decisions required for complex modeling while avoiding
    representation burden that is irrelevant to the task?}

This trade-off is important in CAD.
A fine-grained coding interface gives the model broad control over the underlying software,
but the same CAD operation can often be expressed through many programs,
command sequences, variable names, temporary objects, and API combinations.
These implementations may produce the same geometric state transition,
but they appear as different generation and supervision targets.
The model must then spend capacity on implementation choices that do not affect the intended state transition.
A semantic interface can remove many of these choices by representing the intended operation directly.
However, excessive compression can hide state differences that require
different actions or remove operations needed for correct modeling and error recovery.
A useful harness must therefore reduce irrelevant representation choices
while preserving the actions and observations needed for correct decisions.

To make this trade-off explicit, we formalize the model-facing interface as a
\emph{Harness-Induced CAD Decision Space}: a harness maps model outputs to
executable operations, maps internal CAD states to model-visible observations,
and returns execution results as feedback, thereby determining which
operations the agent can perform, which state differences and failures it can
identify, and which policies it can execute and learn. We then introduce the
\emph{Semantic Decision Bottleneck} as a design principle for this space: it
compresses behaviorally equivalent implementations and removes state
information that does not affect the next decision, while preserving the
actions and observations needed to reach the target geometry and recover from
errors.

Following this principle, we develop \textsc{DeciCAD}, a domain-native harness
for closed-loop CAD agents. \textsc{DeciCAD} represents task-relevant design
transitions as semantic CAD actions, grounds these actions in a persistent CAD
document, and returns structured execution feedback together with
decision-oriented observations. We further develop a
harness-grounded learning pipeline that reconstructs causally conditioned
action supervision from successful and recovered interaction trajectories,
initializes CAD action policies through supervised fine-tuning (SFT),
and supports online reinforcement learning (RL) using policy-generated
trajectories and geometry-grounded rewards.
Figure~\ref{fig:DeciCAD-overview} summarizes
this execution and learning loop.

Across four CAD benchmarks, \textsc{DeciCAD} consistently produces more
accurate final geometry than competitive CAD agents and a general-purpose
coding harness, while reducing model-output tokens by 45--88\% relative to the
coding harness on text-to-CAD tasks. Observation ablations show that showing more CAD state does not always help: unfiltered state inflates the input context, and coarse feedback strips out what the agent needs in order to recover. With matched training settings, open-weight models trained in the \textsc{DeciCAD} decision space learn stronger long-horizon CAD action policies than the same models trained through code-level interaction.

Our contributions are summarized as follows:
\begin{itemize}
    \item We define the \emph{Harness-Induced CAD Decision Space} and introduce the \emph{Semantic Decision Bottleneck} as a design principle for CAD agent interfaces.
          We design \textsc{DeciCAD}, which combines semantic CAD actions, compact observations, structured feedback, and persistent object state.

    \item We develop a harness-grounded policy-learning pipeline that reconstructs causally conditioned supervision from interaction trajectories, initializes policies through SFT, and supports geometry-grounded online RL.

    \item Experiments across four benchmarks verify
          \textsc{DeciCAD} improves geometric accuracy, token efficiency,
          error recovery, and long-horizon policy learnability significantly.
\end{itemize}

\begin{figure*}[t]
    \centering

    \includegraphics[width=0.8\textwidth]{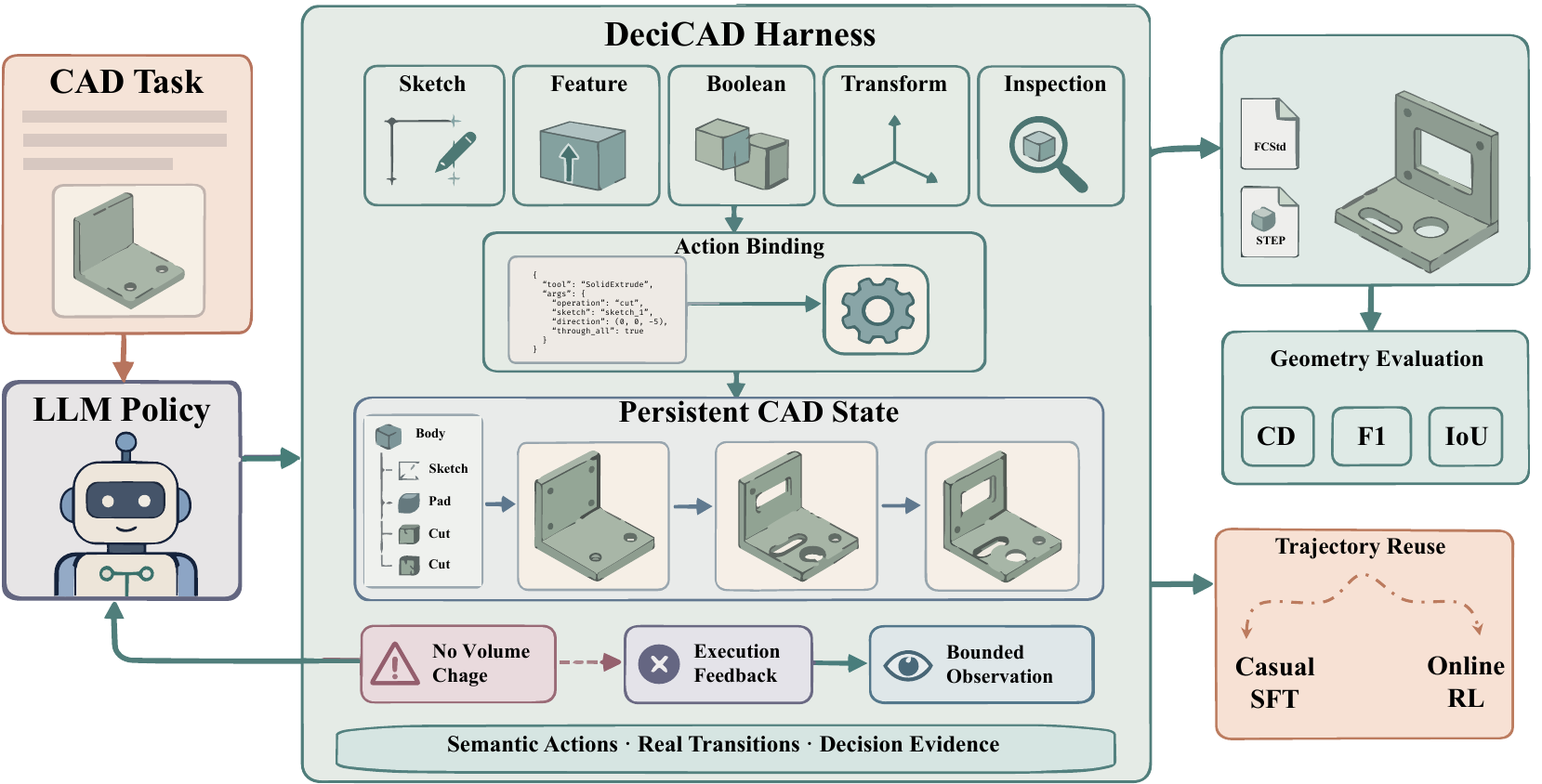}
    \caption{Overview of \textsc{DeciCAD}. Semantic CAD actions are bound to real
        transitions in persistent CAD state, with bounded observations and
        execution feedback; the resulting trajectories support causal SFT and
        geometry-grounded online RL.}
    \label{fig:DeciCAD-overview}
\end{figure*}

\section{Related Work}

Our work connects agent interaction interfaces, decision-preserving
abstraction, and policy learning in interactive CAD environments. We review
these directions through the common question of how an agent's operating
representation affects execution and learning.

\subsection{Agent Interfaces and Decision Abstraction}

Tool-augmented language models connect generation to executable operations
\citep{parisi2022talm,karpas2022mrkl,schick2023toolformer}, while API-Bank,
Gorilla, and ToolLLM extend tool learning to executable dialogues and large
API ecosystems \citep{li2023apibank,patil2023gorilla,qin2023toolllm}. ReAct
interleaves reasoning, actions, and observations \citep{yao2022react};
SWE-agent, WebArena, and OSWorld further show that interaction design affects
behavior in stateful coding, web, and computer environments
\citep{yang2024sweagent,zhou2023webarena,xie2024osworld}. They establish
the empirical importance of interfaces, but generally treat them as tool
provision or runtime infrastructure. We instead study the harness as the joint
action--state representation that defines the policy space presented to the
model.

This view is related to abstraction in sequential decision-making. MDP
homomorphisms and state-abstraction theory characterize when compressed
representations preserve transition, reward, value, or policy structure
\citep{ravindran2003smdp,li2006stateabstraction,abel2016nearoptimal,
    abel2020valuepreserving}; the Information Bottleneck similarly retains
target-relevant information while suppressing irrelevant input dependence
\citep{tishby2000informationbottleneck}. Classical abstractions operate within
an already specified decision process. An LLM harness instead jointly defines
model actions, visible observations, execution bindings, and feedback
semantics. Our Harness-Induced Decision Space describes this transformation,
and the Semantic Decision Bottleneck constrains its compression by decision
preservation.

\subsection{CAD Generation and Interactive Agents}
CAD generation has been studied through construction sequences, executable
programs, and interactive actions. DeepCAD and Fusion 360 Gallery represent
structured modeling histories \citep{wu2021deepcad,willis2020fusion360gallery},
while Text2CAD and CADmium learn parametric command sequences from language
\citep{khan2024text2cad,govindarajan2025cadmium}.
CAD-Coder systems generate
executable CAD code \citep{guan2025cadcoder,doris2025cadcoder}, and CADFusion,
CADSmith, and CADDesigner incorporate rendering, execution repair, or
multi-agent feedback
\citep{wang2025cadfusion,barkley2026cadsmith,fan2025caddesigner}. Code is
expressive, but also makes program organization and API realization part of
the model output.

Interactive systems instead place decisions in evolving design state.
CAD-Assistant combines a Python interpreter with FreeCAD operations
\citep{mallis2024cadassistant}; ToolCAD trains tool-using CAD agents from
demonstrations and online curriculum learning
\citep{gong-etal-2026-toolcad}; CADMate learns sequential assembly actions
using geometric and physical rewards \citep{chen2026cadmate}. Our contribution
is complementary: under fixed models, tasks, and terminal geometry, we compare
the decision spaces induced by domain-semantic actions and a general-purpose
coding harness.

\subsection{Policy Learning from Interaction}

Tool policies commonly acquire initial capability from generated or curated
demonstrations \citep{schick2023toolformer,li2023apibank,qin2023toolllm}. In a
stateful environment, however, action validity depends on the preceding
trajectory and feedback; imitation therefore encounters distribution shift
when student-induced states depart from teacher trajectories
\citep{ross2011reduction}. Online methods address this gap using execution
outcomes on states visited by the learned policy
\citep{feng2025retool,qian2025toolrl}, including curriculum reinforcement
learning in CAD \citep{gong-etal-2026-toolcad}. We align causal supervision,
student rollouts, and environment feedback within the same Harness-Induced
Decision Space, enabling a controlled study of how representation affects both
imitation and online learning.

\section{Method}

The harness determines how the closed CAD loop is presented to the language
model: whether the model writes general-purpose code or selects a semantic CAD
action, which part of the current design is shown, and how execution failures
are reported. We formalize this interaction, formulate the Semantic Decision
Bottleneck, instantiate it in \textsc{DeciCAD}, and use the same
representation for SFT and online RL.

\subsection{From a CAD Backend to a Model-Facing Decision Process}

\paragraph{CAD tasks and backend state.}
We use \(g\) to denote one concrete CAD task, which may contain a natural
language instruction, optional initial state or reference image, and a final
design criterion.
At step \(t\), its internal state
\(s_t\in\mathcal S\) includes the current document, feature structure,
geometry, parameters, and dependencies;
note that this is the CAD system's internal
configuration, rather than text shown to the LLM.
An element
\(u_t\in\mathcal U\) is a backend operation, and
\(
s_{t+1}\sim P(\cdot\mid s_t,u_t)
\)
describes how the document changes after its execution. Task \(g\) specifies
an initial-state distribution \(\mu_g\), a maximum number of interaction steps
\(H\), and reward \(r_g(s_t,u_t,s_{t+1})\). This process describes what the
CAD system can execute. It does not yet describe what the language model is
allowed to say or see.

\begin{figure}[t]
    \centering
    \includegraphics[width=0.8\linewidth]{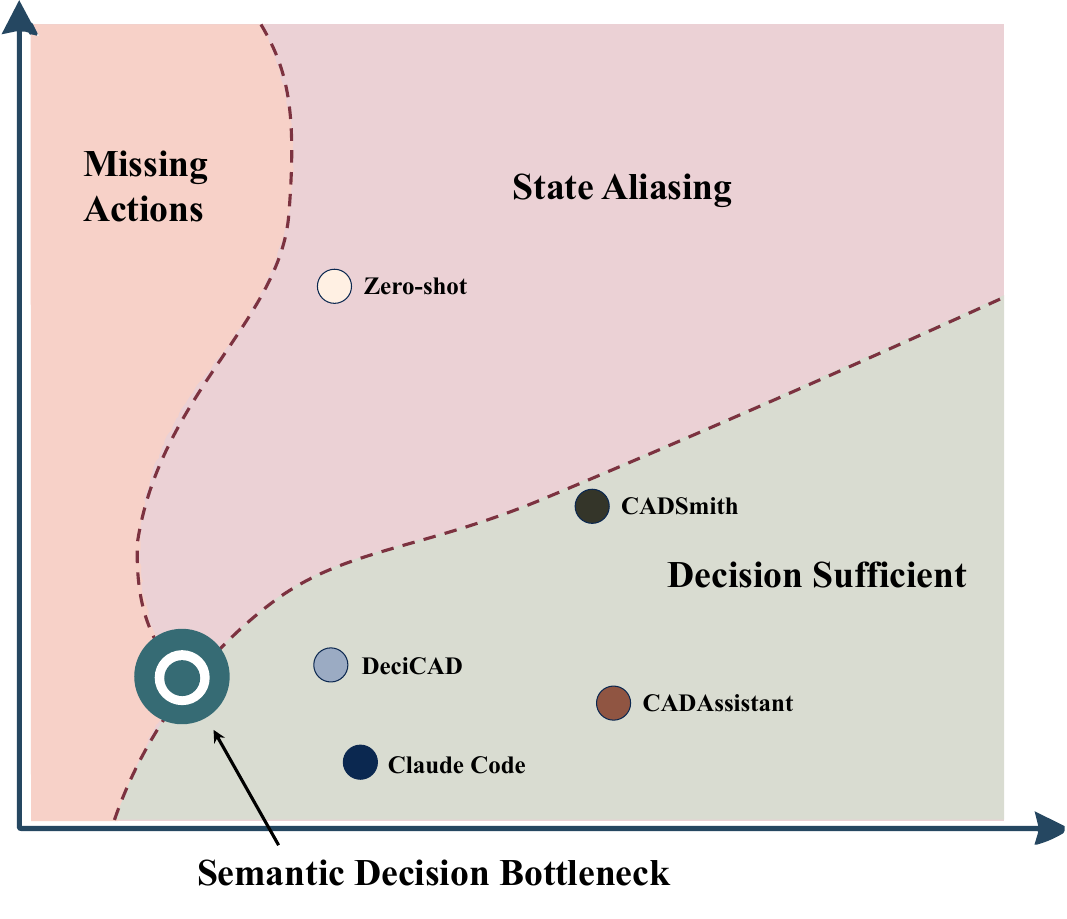}
    \caption{Semantic CAD Decision Bottleneck. Missing actions and state
        aliasing violate action and observation sufficiency, respectively;
        \textsc{DeciCAD} is designed to approach the decision-sufficient
        boundary. Method locations are conceptual and not to scale.}
    \label{fig:semantic-cad-decision-bottleneck}
\end{figure}

\paragraph{What the harness changes.}
A language model does not directly issue an internal backend operation
\(u_t\), nor does it receive the full backend state \(s_t\). A harness provides
the model-facing interface between the two. We represent it as
\begin{equation}
    \label{eq:harness}
    \mathcal H=(\mathcal A_{\mathcal H},
    \beta_{\mathcal H},O_{\mathcal H}),
\end{equation}
whose components have direct operational meanings:
\begin{itemize}
    \item \(\mathcal A_{\mathcal H}\) is the set of admissible model-level actions exposed by the harness. In a coding harness, an action may be a Python program. In
          \textsc{DeciCAD}, it is a typed operation, such as
          \texttt{Extrude(sketch=S3, distance=10mm)}.
    \item \(\beta_{\mathcal H}\) interprets the model output in the current CAD
          document and binds it to an executable backend operation.
    \item \(O_{\mathcal H}\) determines what the model receives after
          execution, including compact descriptions of success, failure, and
          affected design objects.
\end{itemize}
The harness thus translates between model outputs, backend operations, and
model-visible execution evidence without replacing the CAD kernel.

\paragraph{One closed-loop interaction step.}
One interaction step is
\begin{equation}
    \label{eq:harness-transition}
    \begin{aligned}
        a_t     & \sim\pi_\theta(\cdot\mid h_t;\mathcal H), \\
        u_t     & =\beta_{\mathcal H}(a_t,s_t),             \\
        s_{t+1} & \sim P(\cdot\mid s_t,u_t),                \\
        y_{t+1} & \sim O_{\mathcal H}
        (\cdot\mid s_t,a_t,s_{t+1}),
    \end{aligned}
\end{equation}
Here, \(a_t\) is what the model emits, \(u_t\) is what the CAD backend
executes, and \(y_{t+1}\) is what the model sees afterward. The model chooses
its next action from the visible history
\begin{equation*}
    h_t=(g,y_0,a_0,y_1,\ldots,a_{t-1},y_t).
\end{equation*}
This contains the task specification and previous model actions and harness
responses, but not the hidden backend state \(s_t\). Read-only inspection is
also represented as an action, keeping information gathering and
geometry-changing operations in the same closed loop.

Combining backend execution and harness response gives the induced transition
kernel
\begin{equation}
    \label{eq:hids}
    T_{\mathcal H}(s',y\mid s,a)
    =P\!\left(s'\mid s,\beta_{\mathcal H}(a,s)\right)
    O_{\mathcal H}(y\mid s,a,s').
\end{equation}
We call the model-facing interaction defined by
Eqs.~\eqref{eq:harness-transition}--\eqref{eq:hids} the
\emph{Harness-Induced CAD Decision Space}, comprising the actions, responses,
and interaction histories exposed by the harness.

The distinction is important. A coding harness and \textsc{DeciCAD} may call
the same CAD kernel yet induce different sequential prediction problems. The
former requires choices about program structure, API calls, and error
handling; the latter requires choices about CAD operations and their geometric
parameters, with different evidence after each step.

\subsection{Semantic Decision Bottleneck}

The preceding formulation exposes two sources of unnecessary difficulty. On
the output side, many syntactically different programs can realize the same
CAD change. On the input side, the CAD backend contains much more information
than each decision requires. A useful harness should compress both directions
without removing a necessary operation or hiding a distinction that changes
the correct next action. Figure~\ref{fig:semantic-cad-decision-bottleneck}
illustrates this trade-off.

\paragraph{Preserving necessary actions.}
Let \(V^*_{g,t}(s)\) be the best return achievable from backend state \(s\) at
step \(t\), and let \(Q^*_{g,t}(s,u)\) be the return obtained by first
executing backend operation \(u\) and then acting optimally. A harness is
\emph{action-sufficient} if
\begin{equation}
    \label{eq:action-sufficiency}
    \max_{a\in\mathcal A_{\mathcal H}}
    Q^*_{g,t}\!\left(s,\beta_{\mathcal H}(a,s)\right)
    =V^*_{g,t}(s)
\end{equation}
for every task-reachable state. The harness need not expose every low-level
CAD command under Eq.~\eqref{eq:action-sufficiency}, only an action realizing
an optimal next operation when needed.

\paragraph{Preserving necessary state distinctions.}
A harness response need not reveal the backend state, so one visible history
may be consistent with several backend states, as in a partially observable
decision process \citep{kaelbling1998pomdp}.
Define
\begin{equation}
    \label{eq:consistent-states}
    \mathcal B_{\mathcal H}(h_t)
    :=\bigl\{s\in\mathcal S:
    \Pr(s_t=s\mid h_t;\mathcal H)>0\bigr\}
\end{equation}
as the backend states consistent with \(h_t\). We call \(h_t\)
\emph{reachable} when this set is nonempty. For a backend state \(s\), let
\begin{equation}
    \label{eq:optimal-actions}
    \mathcal A^*_{\mathcal H,g,t}(s)
    :=\operatorname*{arg\,max}_{a\in\mathcal A_{\mathcal H}}
    Q^*_{g,t}\!\left(s,\beta_{\mathcal H}(a,s)\right).
\end{equation}
These are the model actions optimal at \(s\). The harness is
\emph{observation-sufficient} if the states in
Eq.~\eqref{eq:consistent-states} share an action from the sets in
Eq.~\eqref{eq:optimal-actions}:
\begin{align}
    \label{eq:observation-sufficiency}
    \bigcap_{s\in\mathcal B_{\mathcal H}(h_t)}
    \mathcal A^*_{\mathcal H,g,t}(s)\neq\varnothing,
    \quad \forall\text{ reachable }h_t,\ t<H.
\end{align}
The harness may map two backend states to the same visible history when the
same next action is optimal in both; it must distinguish them when they require
incompatible next actions. A generic ``operation failed'' message is therefore
insufficient if different failures require different repairs.

A harness that is both action-sufficient and observation-sufficient is
\emph{decision-sufficient}. Any selection from the intersection in
Eq.~\eqref{eq:observation-sufficiency} depends only on \(h_t\) and, by
action sufficiency, is optimal at the true state; backward induction gives
\(\mathbb E_{s_0\sim\mu_g}[V^*_{g,0}(s_0)]\) for the resulting
history-dependent policy. An inexpressible operation therefore violates
action sufficiency, whereas a response merging states that require
incompatible repairs violates observation sufficiency.

\paragraph{Compressing implementation and context.}
Action and observation sufficiency specify what a harness must preserve, but
do not uniquely determine its model-facing representation. Among
decision-sufficient interfaces, we prefer those that reduce unnecessary
burden on both model output and input. We refer to this principle as the
\emph{Semantic Decision Bottleneck}: compressing action representations and
visible context while retaining the distinctions required for correct action
selection and error recovery. \textsc{DeciCAD} instantiates this principle
through domain-specific action schemas, concise execution feedback, and
targeted inspection operations that retrieve additional state only when
needed. A formal preference relation over decision-sufficient harnesses is
provided in the appendix.

\subsection{\textsc{DeciCAD}: A Domain-Native CAD Harness}

\textsc{DeciCAD} implements the preceding design principle through semantic
actions, state-grounded execution, and decision-oriented responses. These
three components correspond directly to
\(\mathcal A_{\mathcal H}\), \(\beta_{\mathcal H}\), and
\(O_{\mathcal H}\) in Eq.~\eqref{eq:harness}.

\subsubsection{Semantic CAD Actions}

\textsc{DeciCAD} represents each model action as
\begin{equation*}
    a_t=(\omega_t,\mathbf{x}_t),
\end{equation*}
where \(\omega_t\) identifies a typed CAD operation and
\(\mathbf{x}_t\) contains its structured arguments. These actions express
design intent such as sketch and
constraint editing, parametric features, Boolean operations, transformations,
inspection, and artifact persistence
\citep{willis2020fusion360gallery,wu2021deepcad,mallis2024cadassistant}.
Using these transitions as the policy-action unit removes
distinctions caused solely by program organization, variable naming, or API
expansion, while typed parameters and state preconditions retain the
information needed for execution.

\subsubsection{State-Grounded Execution}

The binding map \(\beta_{\mathcal H}\) resolves each domain action against the
CAD document that exists at that moment and invokes the corresponding backend
transition. All accepted actions modify the same persistent document, and
stable model-visible identifiers let later actions refer to previously created
entities. If parsing, reference resolution, precondition checking, or backend
execution fails, the harness keeps the document recoverable and reports the
failure. Action semantics are thus grounded in actual CAD transitions rather
than textual well-formedness alone.

\subsubsection{Decision-Oriented Responses}

After executing \(a_t\), \textsc{DeciCAD} returns
\(
y_{t+1}=(e_t,o_{t+1}),
\)
where \(e_t\) summarizes the execution consequence and \(o_{t+1}\) contains
additional state information needed for subsequent decisions. The response
indicates whether the action executed, which objects were modified, which
postconditions hold, and where a failure occurred.
When more information is needed, inspection actions retrieve object hierarchy,
constraints, parameters, dependencies, local measurements, or terminal
artifacts.

A complete state dump repeatedly injects irrelevant context; overly generic
feedback can instead map states requiring different recovery actions to the
same observation. \textsc{DeciCAD} uses compact consequence-related evidence
by default and exposes targeted inspection actions so that the policy acquires
additional evidence only when the next decision requires it. Detailed action schemas, reference
mechanisms, backend adaptation, and runtime isolation are provided in the
appendix.

The resulting \textsc{DeciCAD} interaction loop is
\begin{equation*}
    h_t\xrightarrow{\pi_\theta}a_t
    \xrightarrow{\beta_{\mathcal H},P}s_{t+1}
    \xrightarrow{O_{\mathcal H}}y_{t+1},
\end{equation*}
which concretely instantiates Eq.~\eqref{eq:harness-transition}.

\begin{table*}[t]
    \centering
    \footnotesize
    \setlength{\tabcolsep}{3pt}
    \begin{tabular}{@{}l*{16}{c}@{}}
        \toprule
        & \multicolumn{4}{c}{CADPrompt} & \multicolumn{4}{c}{P3D-Text-to-3D} & \multicolumn{4}{c}{CADInstruct} & \multicolumn{4}{c}{P3D-Assembly-3D} \\
        \cmidrule(lr){2-5} \cmidrule(lr){6-9} \cmidrule(lr){10-13} \cmidrule(lr){14-17}
        Method & CD$\downarrow$ & F1$\uparrow$ & IoU$\uparrow$ & Avail.$\uparrow$ & CD$\downarrow$ & F1$\uparrow$ & IoU$\uparrow$ & Avail.$\uparrow$ & CD$\downarrow$ & F1$\uparrow$ & IoU$\uparrow$ & Avail.$\uparrow$ & CD$\downarrow$ & F1$\uparrow$ & IoU$\uparrow$ & Avail.$\uparrow$ \\
        \midrule
        Zero-shot & 0.32 & 61.84 & \underline{74.10} & 93.50 & 0.35 & 37.11 & 48.15 & 90.00 & 0.93 & 19.28 & \underline{23.24} & \textbf{100.00} & 0.52 & 25.40 & 22.35 & 71.92 \\
        CADSmith & \textbf{0.26} & 63.70 & 72.08 & \textbf{100.00} & \underline{0.33} & \underline{40.74} & 52.73 & \textbf{100.00} & \underline{0.85} & \underline{20.11} & 23.14 & \textbf{100.00} & 0.64 & 26.27 & 21.19 & 96.06 \\
        CAD-Assistant & 0.41 & 56.15 & 65.77 & \textbf{100.00} & 0.45 & 35.83 & \underline{53.04} & \textbf{100.00} & 0.98 & 18.07 & 14.17 & \textbf{100.00} & 0.64 & 25.43 & 21.12 & 99.51 \\
        Claude Code & 0.29 & \underline{65.44} & 73.86 & \textbf{100.00} & 0.34 & 40.44 & 51.20 & \textbf{100.00} & 0.92 & 20.10 & 22.62 & \textbf{100.00} & \textbf{0.32} & \underline{31.95} & \underline{23.22} & \textbf{100.00} \\
        \textsc{DeciCAD} & \underline{0.28} & \textbf{65.83} & \textbf{75.12} & \textbf{100.00} & \textbf{0.25} & \textbf{43.63} & \textbf{54.22} & \textbf{100.00} & \textbf{0.84} & \textbf{23.49} & \textbf{23.35} & \textbf{100.00} & \underline{0.49} & \textbf{32.11} & \textbf{27.07} & \textbf{100.00} \\
        \midrule
        \multicolumn{17}{@{}l}{\textit{Tokens per task (thousands, input / output)}} \\
        Zero-shot & \multicolumn{4}{c}{0.36 / 0.22} & \multicolumn{4}{c}{0.90 / 1.08} & \multicolumn{4}{c}{0.82 / 0.80} & \multicolumn{4}{c}{0.76 / 1.66} \\
        CADSmith & \multicolumn{4}{c}{8.15 / 1.40} & \multicolumn{4}{c}{30.15 / 6.91} & \multicolumn{4}{c}{22.73 / 5.81} & \multicolumn{4}{c}{59.23 / 19.06} \\
        CAD-Assistant & \multicolumn{4}{c}{26.39 / 2.54} & \multicolumn{4}{c}{36.20 / 5.47} & \multicolumn{4}{c}{42.40 / 6.41} & \multicolumn{4}{c}{51.95 / 8.37} \\
        Claude Code & \multicolumn{4}{c}{509.66 / 9.37} & \multicolumn{4}{c}{704.34 / 85.57} & \multicolumn{4}{c}{702.14 / 110.64} & \multicolumn{4}{c}{1399.69 / 23.02} \\
        \textsc{DeciCAD} & \multicolumn{4}{c}{293.61 / 4.27} & \multicolumn{4}{c}{390.54 / 10.54} & \multicolumn{4}{c}{351.13 / 61.38} & \multicolumn{4}{c}{655.06 / 10.31} \\
        \bottomrule
    \end{tabular}
    \caption{Overall performance across four benchmarks, ordered by increasing
        calibrated difficulty. CD, F1, and IoU are scaled by 100; Avail.\ is the
        percentage of tasks producing evaluable geometry. Best results are in
        \textbf{bold} and second-best results are \underline{underlined}. Token counts in the lower
        block are averaged per task, are descriptive, and are not ranked.}
    \label{tab:overall_performance}
\end{table*}

\subsection{Harness-Grounded Policy Learning}

Execution and learning use the same CAD decision space. Teacher supervision,
student rollouts, and environment optimization retain the same action
semantics, persistent references, and feedback protocol. For each accepted
teacher trajectory under harness \(\mathcal H\), we reconstruct the causal
history \(h_{n,t}\) visible immediately before action \(a_{n,t}\). The
resulting pairs form
\(\mathcal D_{\mathrm{SFT}}^{\mathcal H}
=\{(h_{n,t},a_{n,t})\}_{n,t}\), and we optimize
\begin{equation*}
    \mathcal L_{\mathrm{SFT}}(\theta)
    =-\mathbb E_{(h,a)\sim\mathcal D_{\mathrm{SFT}}^{\mathcal H}}
        [\log\pi_\theta(a\mid h;\mathcal H)].
\end{equation*}
Later feedback, actions, and terminal evaluation never enter the input,
preserving causal conditioning in sequential imitation
\citep{ross2011reduction,qin2023toolllm}. Recovered failures are retained
because their error state, feedback, and corrective action provide genuine
CAD recovery supervision.

Uniform decision sampling would overweight long trajectories, repeated
inspection, and frequent operations. We therefore stratify causal decisions
by task difficulty, design stage, preceding execution outcome, and target
action type, preserving modeling, inspection, and recovery decisions under
their original histories.

SFT covers teacher-visited CAD states, whereas closed-loop execution reaches
states induced by student actions.
We therefore sample current-policy rollouts by the same harness
and optimize groups containing evaluable terminal geometry.
Geometry quality supplies the primary group-relative signal;
execution and protocol evidence describe the trajectory but cannot make an
all-failure group trainable. Reference-regularized Group Relative Policy
Optimization (GRPO) with curriculum sampling instantiates this update
\citep{bengio2009curriculum,shao2024deepseekmath,qian2025toolrl,
    gong-etal-2026-toolcad}; reward coefficients, sampling schedules, and optimizer
settings are deferred to the appendix.

\section{Experiments}

Our evaluation follows the methodological chain developed above. We ask
whether a semantic CAD decision space preserves the capabilities required for
complex modeling, whether its benefits are associated with lower
task-irrelevant representation burden, and whether different harness-induced
policy spaces exhibit different finite-data learnability and online
improvement.

\subsection{Experimental Setup}

\paragraph{Benchmark.}
We evaluate harness-level capability, in increasing calibrated difficulty, on
CADPrompt \citep{alrashedy2025cadcode}, a hard subset of
CADInstruct \citep{lv2025cadinstruct}, and two P3D-derived benchmarks:
P3D-Text-to-3D from Text2CAD and the multimodal P3D-Assembly-3D from the
Fusion 360 Gallery
\citep{khan2024text2cad,willis2020fusion360gallery}.

\paragraph{Difficulty.}
We assign each task a calibrated difficulty that combines semantic complexity,
reference-mesh complexity, and description ambiguity, and divide each
benchmark into seven equal-frequency quantiles \(Q1\)--\(Q7\). Quantiles
therefore express relative difficulty within a benchmark rather than the same
absolute complexity across benchmarks. The full definition and the
distribution of each benchmark are given in Appendix~D.1 and
Figure~\ref{fig:supp-difficulty-distribution}.

\paragraph{Baseline.}
Using identical task inputs and Claude Opus 4.7
\citep{anthropic2026opus47}, \textsc{DeciCAD} is
compared with zero-shot CadQuery generation, CADSmith, CAD-Assistant, and
Claude Code \citep{anthropic2026claudecode}.
We study policy learning on the standard train--test split of Text2CAD/DeepCAD
\citep{khan2024text2cad,wu2021deepcad}, fixing
Qwen3.5-9B \citep{qwen2026qwen35} as the student and Claude Sonnet 4.6
\citep{anthropic2026sonnet46} as the teacher in both the Coding and
\textsc{DeciCAD} decision spaces.
Generated and reference meshes
are normalized, center-aligned, and rigidly registered.

\paragraph{Metric.}
We report Chamfer Distance (CD), surface F1, volumetric Intersection over Union (IoU), and
Geometry Availability. F1, IoU, and CD are computed over valid assets;
Geometry Availability measures the percentage of tasks entering geometric
evaluation. Harness-level and ablation tables scale CD, F1, and IoU by 100,
whereas policy-learning tables report their original values. These
point-set and occupancy-based measures follow established 3D reconstruction
evaluation practice \citep{fan2017pointset,mescheder2019occupancy}. To measure
representation cost, we report model-visible input and generated output tokens
averaged per task; provider-specific cache statistics are excluded. Across
tables, best results are bold and second best are underlined.
Policy learning for both \textsc{DeciCAD} and the coding harness used eight
NVIDIA H800 GPUs; complete configurations are provided in the appendix.

\begin{figure*}[t]
    \centering
    \includegraphics[width=\textwidth]{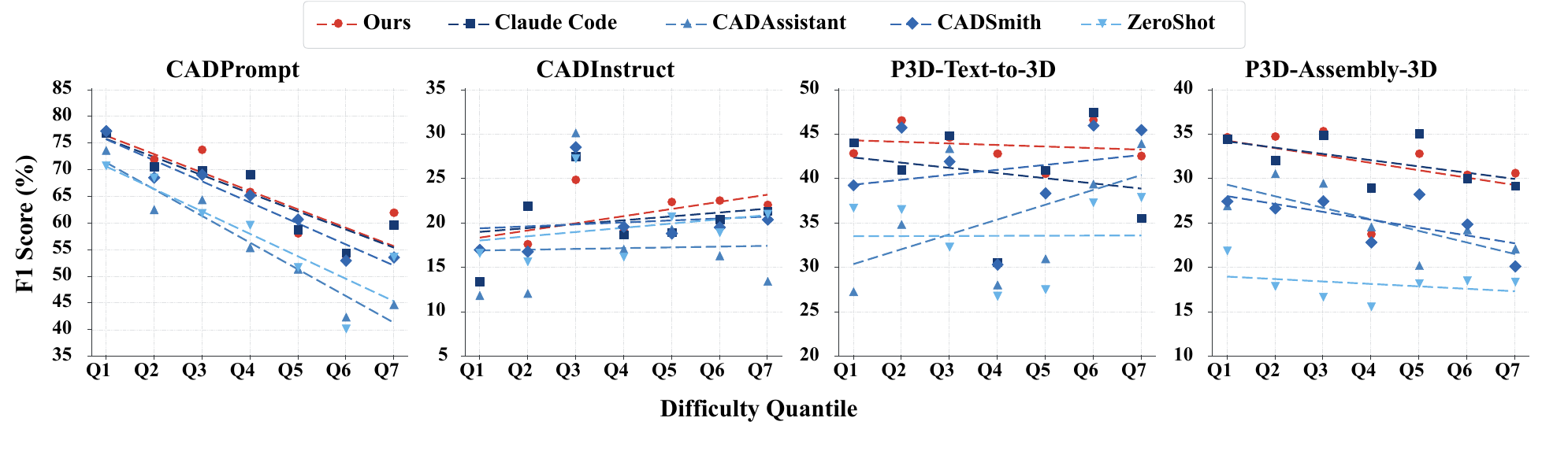}
    \caption{Linear summaries of end-to-end F1 across difficulty quantiles. Dashed lines
        are least-squares fits over the seven quantiles.}
    \label{fig:supp-f1-slopes}
\end{figure*}

\subsection{Decision Sufficiency and CAD Capability}

Table~\ref{tab:overall_performance} compares the same frontier model under
different CAD representations and interaction mechanisms. \textsc{DeciCAD}
matches or exceeds the strongest coding harness in geometric fidelity while
producing evaluable assets across all four benchmarks.

To check that this advantage is not driven only by easy tasks, we divide each
benchmark into seven calibrated difficulty quantiles and report end-to-end F1,
which assigns zero to missing or invalid outputs
(Figure~\ref{fig:supp-f1-slopes}). \textsc{DeciCAD} and Claude Code have
similar slopes on CADPrompt and both decline slowly on P3D-Assembly-3D;
\textsc{DeciCAD} is flatter on P3D-Text-to-3D, indicating a more stable
advantage under complex textual constraints. Across the individual cells,
\textsc{DeciCAD} leads 14 of the 28 dataset--quantile combinations and exceeds
Claude Code at the hardest quantile of every benchmark, where missing
operations and unresolved state ambiguity would be most visible. Semantic CAD
actions therefore preserve the transitions needed for complex modeling while
supporting more reliable final geometry. The full quantile heatmaps and
difficulty distributions are provided in the appendix.

\begin{figure*}[t]
    \centering
    \begin{subfigure}[t]{0.24\textwidth}
        \includegraphics[width=\linewidth]{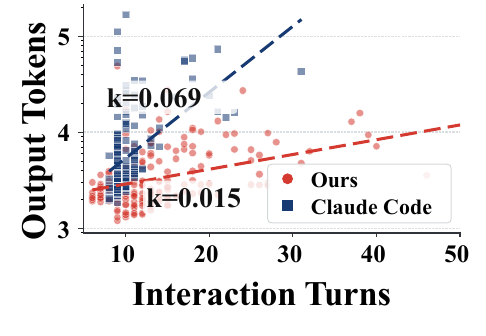}
        \caption{CADPrompt}
    \end{subfigure}\hfill
    \begin{subfigure}[t]{0.24\textwidth}
        \includegraphics[width=\linewidth]{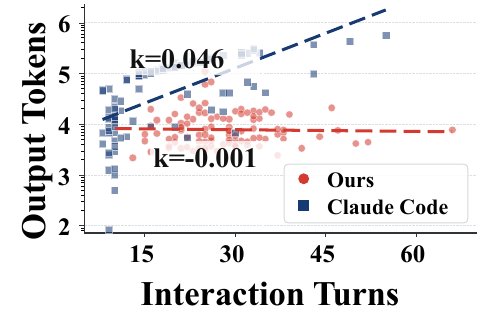}
        \caption{P3D-Text-to-3D}
    \end{subfigure}\hfill
    \begin{subfigure}[t]{0.24\textwidth}
        \includegraphics[width=\linewidth]{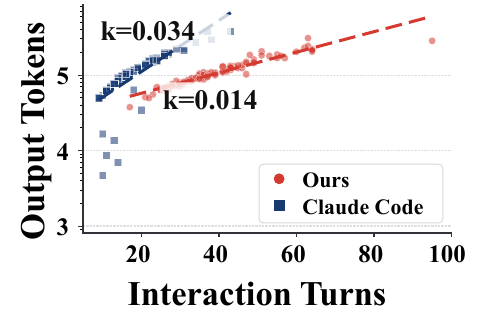}
        \caption{CADInstruct}
    \end{subfigure}\hfill
    \begin{subfigure}[t]{0.24\textwidth}
        \includegraphics[width=\linewidth]{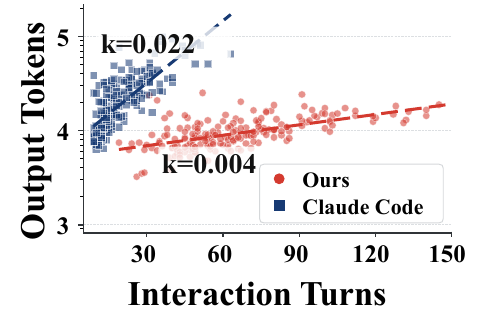}
        \caption{P3D-Assembly-3D}
    \end{subfigure}
    \caption{Interaction length versus cumulative output tokens. Each point is
        a complete task trajectory; dashed lines are log-linear fits.}
    \label{fig:supp-turns-output}
\end{figure*}

\subsection{Validating the Semantic Decision Bottleneck}

We test both sides of the Semantic Decision Bottleneck: output scaling
characterizes action representation, while observation ablations test whether
compression preserves decision-relevant state.

\paragraph{Action-side representation.}
Token totals in Table~\ref{tab:overall_performance} depend on trajectory
length, feedback history, and caching, and therefore do not by themselves
identify the source of the output reduction.
Figure~\ref{fig:supp-turns-output} compares interaction turns with cumulative
output at the task level. Claude Code exhibits consistently positive
log-linear slopes of approximately \(0.022\)--\(0.069\), whereas
\textsc{DeciCAD}'s slopes range from approximately \(-0.001\) to \(0.015\):
its output remains comparatively stable as trajectories lengthen. The
reduction does not arise from early termination. \textsc{DeciCAD}
trajectories may still be long, but each prediction contains an action and
its arguments rather than an expanded implementation, which identifies action
representation, rather than shortened reasoning, as the source of the
efficiency gain.

\begin{table}[h]
    \centering
    \footnotesize
    \setlength{\tabcolsep}{4pt}
    \begin{tabular}{@{}lcccc@{}}
        \toprule
        Metric & SDO & SDO-CFF & RSO & EFO \\
        \midrule
        Avail.\ (\%)$\uparrow$ & \textbf{100.00} & \underline{95.00} & 55.00 & 87.00 \\
        CD$\downarrow$ & \textbf{0.25} & \underline{0.44} & 0.99 & 0.61 \\
        F1$\uparrow$ & \textbf{43.63} & \underline{32.26} & 21.62 & 31.75 \\
        IoU$\uparrow$ & \textbf{54.22} & \underline{43.78} & 33.06 & 40.71 \\
        Input tokens$\downarrow$ & \textbf{390.54} & \underline{3,050.00} & 3,450.00 & 3,200.00 \\
        \bottomrule
    \end{tabular}
    \caption{Ablation of observation and feedback design on P3D-Text-to-3D.
        Variants are Structured Decision-Oriented Observation (SDO; full), SDO
        with Coarse Failure Feedback (SDO-CFF), Raw State Observation (RSO), and
        Execution-Feedback-Only Observation (EFO). CD, F1, and IoU are scaled by
        100; input tokens are in thousands per task.}
    \label{tab:observation_ablation}
\end{table}

\paragraph{Observation-side representation.}
Table~\ref{tab:observation_ablation} isolates observation sufficiency on
P3D-Text-to-3D. More raw state degrades geometry despite increasing context,
whereas removing active inspection or coarsening failure semantics hides
distinctions needed by subsequent actions. Effective compression therefore
retains decision-identifying evidence rather than simply minimizing or
maximizing observation volume. Together, both sides support a
decision-sufficient space with lower representation burden.

\begin{figure*}[t]
    \centering
    \includegraphics[width=0.92\textwidth]{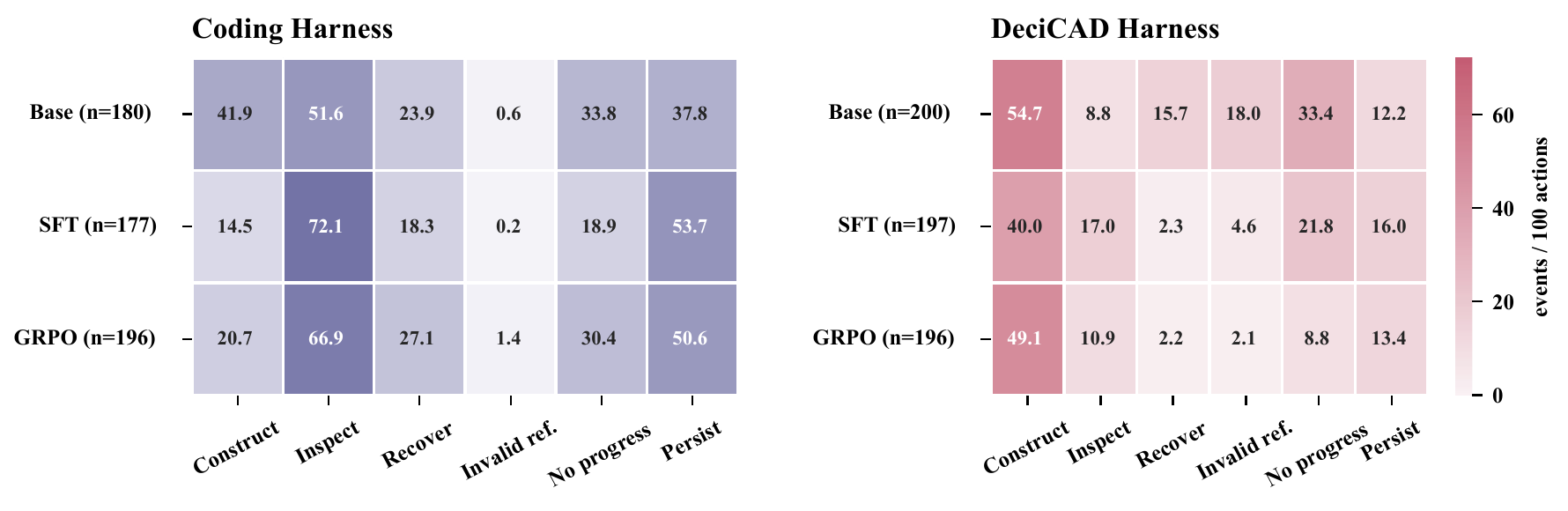}
    \caption{Observable policy behavior across Base, SFT, and GRPO. The left
        and right panels show the Coding and \textsc{DeciCAD} harnesses. Columns
        report events per 100 decisions under a shared behavioral mapping; \(n\)
        is the number of trajectories containing parseable action records.}
    \label{fig:supp-policy-behavior}
\end{figure*}

\subsection{Harness-Grounded CAD Policy Learning}

Frontier models can operate both harnesses, but the corresponding policies
need not be equally learnable by a smaller student. We reconstruct causal
decision supervision from harness-matched teacher trajectories and compare
Qwen3.5-9B policies on the complete Text2CAD test set.

\paragraph{Policy acquisition from interaction trajectories.}
Table~\ref{tab:sft_learnability} compares Base and SFT policies within each
induced decision space.

\begin{table}[h]
    \centering
    \footnotesize
    \setlength{\tabcolsep}{4pt}
    \begin{tabular}{@{}llcccc@{}}
        \toprule
        Harness & Policy & Avail.\ (\%) & F1 & IoU & CD \\
        \midrule
        \multirow{2}{*}{Coding} & Base & 12.5 & 0.5902 & 0.5871 & \underline{0.00972} \\
        & SFT & 60.5 & 0.6460 & 0.6286 & 0.01130 \\
        \midrule
        \multirow{2}{*}{\textsc{DeciCAD}} & Base & \underline{61.5} & \underline{0.6883} & \underline{0.6487} & 0.01023 \\
        & SFT & \textbf{66.5} & \textbf{0.7437} & \textbf{0.6631} & \textbf{0.00782} \\
        \bottomrule
    \end{tabular}
    \caption{Closed-loop policy acquisition on the Text2CAD/\allowbreak DeepCAD test set.
        Availability is computed over all tasks; F1, IoU, and CD are computed
        over valid assets.}
    \label{tab:sft_learnability}
\end{table}

SFT substantially improves the coding policy, showing that the student can
acquire program construction, execution, and repair behavior. Under matched
supervision, however, the \textsc{DeciCAD} policy produces more evaluable
assets and better geometry across all reported metrics. Coding supervision
must resolve geometric intent together with program structure, API choice, and
implementation form, whereas \textsc{DeciCAD} binds supervision directly to
state-conditioned domain actions. The result supports harness-dependent
learnability under finite supervision.

\paragraph{Online improvement from environment feedback.}
Starting
from each SFT policy, we sample student rollouts in its own induced decision
space, and apply reference-regularized GRPO with geometry-grounded rewards.
The student model, task exposure,
samples per task, and evaluator are matched. Improvement is judged by the
frozen unseen-task evaluation rather than training reward, trajectory length,
or token consumption.

\begin{table}[h]
    \centering
    \footnotesize
    \setlength{\tabcolsep}{4pt}
    \begin{tabular}{@{}llcccc@{}}
        \toprule
        Harness & Policy & Avail.\ (\%) & F1 & IoU & CD \\
        \midrule
        \multirow{2}{*}{Coding} & SFT & 60.5 & 0.6460 & 0.6286 & 0.01130 \\
        & SFT+GRPO & 65.5 & 0.6253 & 0.5987 & 0.02253 \\
        \midrule
        \multirow{2}{*}{\textsc{DeciCAD}} & SFT & \underline{66.5} & \underline{0.7437} & \underline{0.6631} & \underline{0.00782} \\
        & SFT+GRPO & \textbf{73.5} & \textbf{0.7880} & \textbf{0.7036} & \textbf{0.00570} \\
        \bottomrule
    \end{tabular}
    \caption{Online policy improvement under matched training and evaluation.
        F1, IoU, and CD are computed over valid assets.}
    \label{tab:grpo_learnability}
\end{table}

Table~\ref{tab:grpo_learnability} shows that online feedback improves
\textsc{DeciCAD} across availability and all three geometric metrics, whereas
the coding policy does not realize the same gains under the matched inference
budget, supporting harness-dependent online learnability.

\paragraph{What the policies learn.}
To compare the learning paths, we map the external effects of both harnesses
into six shared categories: Construct, Inspect, Recover, Invalid Reference, No
Progress, and Persist, and report events per 100 decisions
(Figure~\ref{fig:supp-policy-behavior}). Coding SFT increases Inspect from
51.6 to 72.1 events per 100 decisions, while Construct and Recover both
decline. GRPO restores part of the construction and recovery behavior, but No
Progress rises from 18.9 to 30.4, showing that more repair attempts do not
consistently produce geometric progress. In \textsc{DeciCAD}, SFT reduces
Invalid Reference and No Progress from 18.0 and 33.4 to 4.6 and 21.8, and GRPO
further reduces them to 2.1 and 8.8 while retaining dense construction
behavior. Coding supervision therefore mainly changes execution, inspection,
and repair behavior, whereas \textsc{DeciCAD} learning more directly changes
CAD state transitions and suppresses interactions that fail to advance the
model.

\section{Conclusion}

We formalized Harness-Induced CAD Decision Spaces and instantiated the
Semantic Decision Bottleneck in \textsc{DeciCAD} through semantic actions,
persistent state, and decision-oriented feedback. Across execution, SFT, and
online RL, \textsc{DeciCAD} preserves complex modeling capability while
reducing irrelevant burden and enabling more learnable and optimizable
policies. Together, these findings show that a harness determines both the CAD
operations a model can execute and the distinctions its policy must learn,
while environment feedback improves that policy within the same decision
representation. This establishes harness design as central to CAD agents.
Our study has limitations. The harness comparison uses one CAD backend and
one frontier model, policy learning uses one student model on one dataset,
and mesh-based metrics do not capture editability, constraint satisfaction,
or feature-tree validity. Extending the comparison to further backends,
models, and CAD-native correctness measures is left to future work.

\clearpage
\bibliography{aaai2027}

\clearpage

\appendix
\lstset{style=suppcode}
\section{Overview of the Appendix}

This appendix complements the main text with the theoretical,
implementation, experimental, and system details needed to connect the
Harness-Induced CAD Decision Space to its realization in \textsc{DeciCAD}. The
remainder of this document is organized as follows:

\begin{enumerate}
    \item Appendix~B completes the theoretical analysis of decision
          sufficiency and the Semantic Decision Bottleneck, including the
          conditions under which action and observation compression preserve
          optimal CAD decisions.
    \item Appendix~C explains how \textsc{DeciCAD} instantiates these
          conditions through semantic CAD actions, state grounding, compact
          execution feedback, and active hierarchical inspection.
    \item Appendix~D provides the complete evaluation protocol and additional
          evidence for harness comparison, causal SFT, and geometry-grounded
          online RL.
    \item Appendix~E presents \textsc{DeciCAD} as an interactive Agentic CAD
          system, covering its user workflow, multi-agent control plane,
          backend SDK, the official FreeCAD integration, and the supported backends.
\end{enumerate}

Together, these appendices connect the formal decision-space argument to its
software realization and to the empirical evidence reported in the main text.

\section{Theoretical Foundations}

The main text models a harness as the decision interface between a language
model and a CAD backend. It characterizes decision sufficiency through the
actions that the interface can express and the state distinctions retained by
its observations, connecting partial observability with value-preserving state
abstraction \citep{kaelbling1998pomdp,ravindran2003smdp,
    li2006stateabstraction,abel2016nearoptimal}. This section completes the corresponding theoretical
argument and then formalizes the compression preference underlying the
Semantic Decision Bottleneck. Figure~\ref{fig:supp-bottleneck-detail}
summarizes the resulting criterion: interface compression is desirable only
within the decision-sufficient region.

\begin{figure*}[t]
    \centering
    \includegraphics[width=0.8\textwidth]{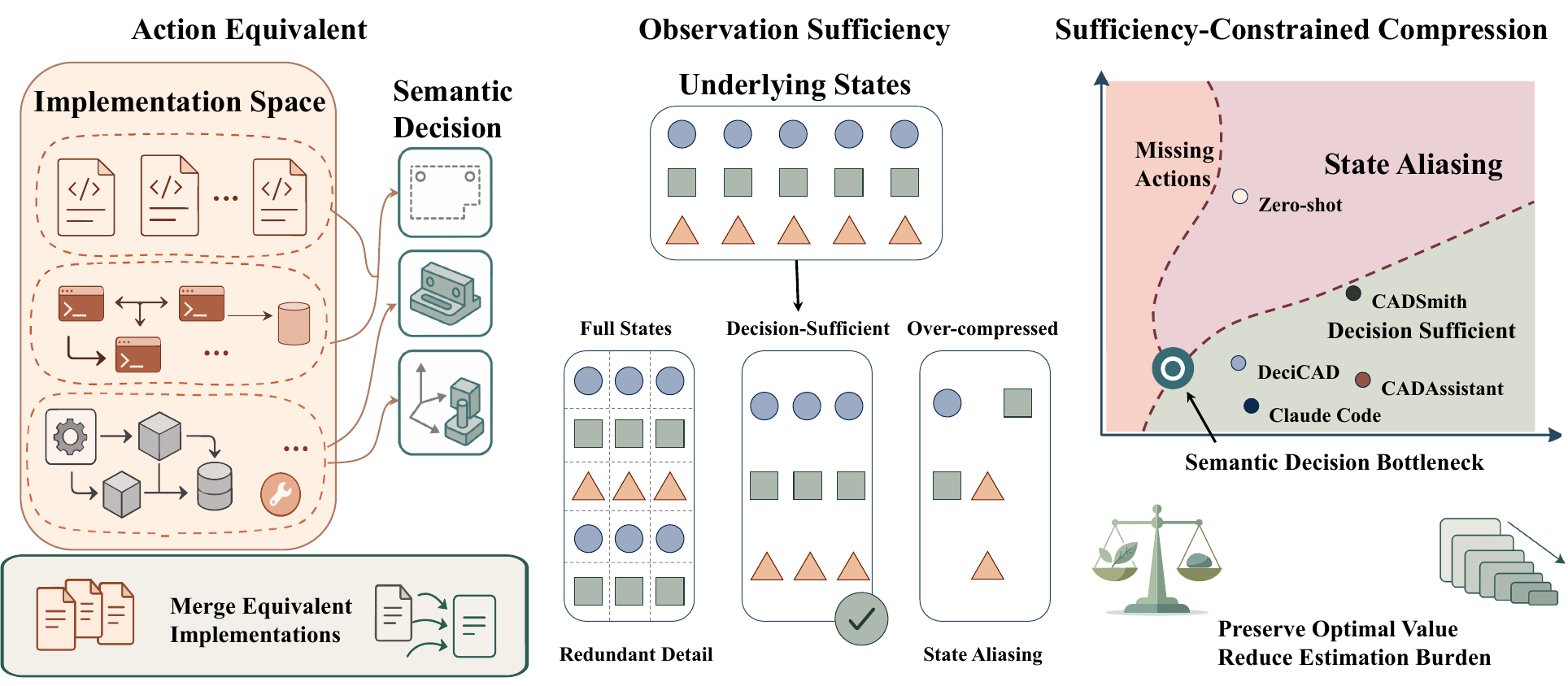}
    \caption{The Semantic Decision Bottleneck compares harnesses by their decision
        capabilities while preserving optimal CAD performance.}
    \label{fig:supp-bottleneck-detail}
\end{figure*}

\subsection{Decision Sufficiency}
\label{sec:supp-decision-sufficiency}

We first restate the objects needed for the analysis. A finite-horizon CAD task
\(g\) specifies an initial-state distribution \(\mu_g\), a horizon \(H\), and
a reward \(r_g(s_t,u_t,s_{t+1})\). The CAD backend has state
\(s_t\in\mathcal S\), admits backend operations \(u_t\in\mathcal U\), and
evolves according to the standard controlled-transition formalism
\citep{sutton2018reinforcement}:
\begin{equation}
    \label{eq:supp-backend-transition}
    s_{t+1}\sim P(\cdot\mid s_t,u_t).
\end{equation}
A state is \emph{task-reachable} if it can be reached from
\(s_0\sim\mu_g\) within \(H\) steps by operations admissible for the task.

The model does not act on this backend process directly. A harness
\begin{equation}
    \label{eq:supp-harness}
    \mathcal H
    =(\mathcal A_{\mathcal H},\beta_{\mathcal H},O_{\mathcal H})
\end{equation}
exposes model actions \(a_t\in\mathcal A_{\mathcal H}\), binds them to backend
operations through \(\beta_{\mathcal H}\), and returns model-visible responses
through \(O_{\mathcal H}\). One closed-loop step is
\begin{equation}
    \label{eq:supp-harness-transition}
    \begin{aligned}
        a_t     & \sim\pi_\theta(\cdot\mid h_t;\mathcal H), \\
        u_t     & =\beta_{\mathcal H}(a_t,s_t),             \\
        s_{t+1} & \sim P(\cdot\mid s_t,u_t),                \\
        y_{t+1} & \sim O_{\mathcal H}
        (\cdot\mid s_t,a_t,s_{t+1}),
    \end{aligned}
\end{equation}
where the policy observes only
\begin{equation}
    \label{eq:supp-visible-history}
    h_t=(g,y_0,a_0,y_1,\ldots,a_{t-1},y_t).
\end{equation}

Let \(V^*_{g,t}(s)\) be the optimal return available from backend state \(s\)
at step \(t\), and \(Q^*_{g,t}(s,u)\) the return from first applying backend
operation \(u\) and then acting optimally. The harness is
\emph{action-sufficient} when
\begin{equation}
    \label{eq:supp-action-sufficiency}
    \max_{a\in\mathcal A_{\mathcal H}}
    Q^*_{g,t}\!\left(s,\beta_{\mathcal H}(a,s)\right)
    =V^*_{g,t}(s)
\end{equation}
at every task-reachable state. This condition requires the model-facing action
space to retain an executable realization of an optimal backend operation, not
every operation or every surface form.

Because \(s_t\) is hidden from the policy, define the states consistent with a
visible history as
\begin{equation}
    \label{eq:supp-consistent-states}
    \mathcal B_{\mathcal H}(h_t)
    :=\left\{s\in\mathcal S:
    \Pr(s_t=s\mid h_t;\mathcal H)>0\right\}.
\end{equation}
A history is \emph{reachable} when this set is nonempty. The model actions that
are optimal at a particular backend state are
\begin{equation}
    \label{eq:supp-optimal-actions}
    \mathcal A^*_{\mathcal H,g,t}(s)
    :=\operatorname*{arg\,max}_{a\in\mathcal A_{\mathcal H}}
    Q^*_{g,t}\!\left(s,\beta_{\mathcal H}(a,s)\right).
\end{equation}
We assume that this set is nonempty at every task-reachable state. The harness
is \emph{observation-sufficient} when
\begin{equation}
    \label{eq:supp-observation-sufficiency}
    \bigcap_{s\in\mathcal B_{\mathcal H}(h_t)}
    \mathcal A^*_{\mathcal H,g,t}(s)\neq\varnothing
\end{equation}
for every reachable \(h_t\) and \(t<H\). Thus, hidden states may share a
visible representation precisely when they retain a common optimal next
action. A harness satisfying
Eqs.~\eqref{eq:supp-action-sufficiency} and
\eqref{eq:supp-observation-sufficiency} is
\emph{decision-sufficient}.

\begin{theorem}[Decision sufficiency]
    \label{thm:decision-sufficiency}
    If a harness \(\mathcal H\) is action-sufficient at every task-reachable state
    and observation-sufficient at every reachable history, then there exists a
    policy \(\pi^*_{\mathcal H}\) depending only on the model-visible history
    \(h_t\) that attains \(V^*_{g,t}(s_t)\) at every reachable decision
    point. Its expected return from the initial-state distribution is therefore
    \begin{equation*}
        \mathbb E_{s_0\sim\mu_g}
        \left[V^*_{g,0}(s_0)\right].
    \end{equation*}
    Thus, the harness preserves the optimal return available through the CAD
    backend.
\end{theorem}

\begin{proof}
    For any reachable history \(h_t\),
    Eq.~\eqref{eq:supp-observation-sufficiency} permits selecting an action
    \(a^*(h_t)\) from the common intersection. The true backend state
    \(s_t\) is necessarily an element of
    \(\mathcal B_{\mathcal H}(h_t)\); hence \(a^*(h_t)\) is optimal for every
    backend state consistent with what the model has observed. Action sufficiency
    in Eq.~\eqref{eq:supp-action-sufficiency} then ensures that the bound operation
    \(\beta_{\mathcal H}(a^*(h_t),s_t)\) attains \(V^*_{g,t}(s_t)\). Backward
    induction from the terminal step shows that the history-dependent policy
    defined by \(a^*(h_t)\) attains the optimal continuation value at every
    reachable decision point. Taking the expectation over
    \(s_0\sim\mu_g\) gives the expected value stated above.
\end{proof}

The theorem clarifies what a model-facing CAD interface must retain. The model
does not require direct access to the complete backend state or to every
low-level command. It requires an executable representation of an optimal
operation whenever one is needed, together with observations that distinguish
backend states whenever those states demand incompatible next actions.

The two conditions also describe distinct failure boundaries. If action
sufficiency fails at a reachable state, no action exposed by the harness can
attain \(V^*_{g,t}(s)\); increasing model capacity or adding policy data cannot
recover an operation absent from the interface. If observation sufficiency
fails at a reachable history, the states consistent with that history share no
optimal model action. Any policy conditioned only on the visible history must
therefore be suboptimal for at least one of them. The resulting expected loss
depends on how often the aliased states occur and on the value gap between the
chosen and optimal operations.

Missing actions and state aliasing consequently constrain different sides of
the harness-induced decision space. The former limits what the policy can
execute; the latter limits what it can distinguish. These boundaries determine
which action and observation details may be compressed while preserving the
optimal CAD policy.

\subsection{Semantic Decision Bottleneck}
\label{sec:supp-semantic-decision-bottleneck}

Decision sufficiency specifies what a harness must preserve, but does not
uniquely determine how those capabilities are represented to the model. Two
harnesses may connect to the same CAD backend and support the same task-level
operations while exposing different numbers of implementation choices,
different state granularity, and different forms of execution feedback. The
Semantic Decision Bottleneck compares these decision-sufficient
representations, adapting the compression principle of the information
bottleneck to a model-facing action--observation interface
\citep{tishby2000informationbottleneck}.

\paragraph{Decision-preserving compression.}
We call a transformation from harness \(\mathcal H_2\) to harness
\(\mathcal H_1\) a \emph{decision-preserving compression} when it satisfies
two conditions. On the action side, model actions may be merged only when, at
every task-reachable state where they are admissible, their bound backend
operations induce the same distribution over next states and the same task
reward for every transition of nonzero probability. The compressed harness
must also retain an executable representation of every consequence needed to
attain Eq.~\eqref{eq:supp-action-sufficiency}. This removes differences in
syntax or implementation without changing a decision-relevant CAD
consequence. On the observation side, visible histories may be identified only
when all backend states consistent with any of those histories still share an
optimal action in the compressed harness. This is precisely the criterion in
Eq.~\eqref{eq:supp-observation-sufficiency}, evaluated after the histories are
identified. The action condition retains an optimal executable consequence at
every task-reachable state, while the observation condition retains a common
optimal action for every compressed history. They therefore preserve action
and observation sufficiency, respectively.

For decision-sufficient harnesses, we write
\begin{equation}
    \label{eq:sdb-preference}
    \mathcal H_1\preceq_{\mathrm{SDB}}\mathcal H_2
\end{equation}
if \(\mathcal H_1\) can be obtained from \(\mathcal H_2\) through a
decision-preserving compression. The relation expresses a representation
preference under fixed decision capability: \(\mathcal H_1\) does not require
the model to distinguish action implementations or visible states that are
irrelevant to optimal CAD decisions in \(\mathcal H_2\).

Decision-preserving compressions can be applied successively: each step
retains the action consequences and state distinctions required by
Eqs.~\eqref{eq:supp-action-sufficiency} and
\eqref{eq:supp-observation-sufficiency}. The relation therefore supports
comparing harnesses by removable semantic structure without selecting a
particular action syntax or observation format as canonical.

\paragraph{Compression frontier.}
The Semantic Decision Bottleneck is the decision-sufficient frontier of this
relation. A harness lies on the frontier when every further semantic
compression violates at least one sufficiency condition. Further action
compression crosses this frontier when it removes the only
representation of an operation required by
Eq.~\eqref{eq:supp-action-sufficiency}. Further observation compression
crosses it when the common-action intersection becomes empty, aliasing states
that demand incompatible next actions. The relation need not select a unique
frontier harness: different action vocabularies or observation organizations
may be incomparable while preserving the same optimal CAD capability. The
bottleneck is therefore neither a unique interface nor the shortest one in
characters or tokens; it is the frontier of semantic compression subject to
preserving the optimal decision policy.

Equation~\eqref{eq:sdb-preference} concerns semantic redundancy rather than
token count itself. In a language-model interface, redundant action
implementations and state descriptions often increase model input or output,
so token use provides an observable measure of representation burden.
Nevertheless, semantic compression and token compression are not equivalent.
We therefore evaluate a harness jointly by task performance, observation
ablations, and token consumption. Nor does the relation by itself guarantee
easier optimization for every finite model; the matched SFT and online RL
experiments in the main text test whether the induced representation improves
policy learning in practice.

\paragraph{Multi-objective harness optimization.}
A natural direction for future work is to treat harness construction as an
explicit multi-objective optimization problem. Let
\(L_{\mathrm{in}}\) and \(L_{\mathrm{out}}\) denote the numbers of
model-facing input and output tokens, respectively, and let
\(N_{\mathrm{inspect}}\) denote the number of targeted inspection operations.
One may seek a Pareto-efficient harness according to
\begin{equation}
    \label{eq:sdb-multiobjective}
    \min_{\mathcal H\in\mathfrak H_{\mathrm{suff}}}
    \mathbf C(\mathcal H)
    =
    \left(
    \mathbb E[L_{\mathrm{in}}],
    \mathbb E[L_{\mathrm{out}}],
    \mathbb E[N_{\mathrm{inspect}}]
    \right),
\end{equation}
where \(\mathfrak H_{\mathrm{suff}}\) is the set of harnesses satisfying both
action and observation sufficiency, and the expectations are taken over the
task distribution and policy-induced trajectories. Different deployment
settings may scalarize these objectives using different weights: for example,
interactive systems may penalize inspection latency more heavily, whereas
resource-constrained deployments may prioritize context and generation
length. More broadly, the harness and policy could be optimized jointly,
subject to semantic sufficiency constraints. We leave automated exploration
of this harness--policy Pareto frontier to future work.

\section{Implementation Details}

The main text characterizes a decision-sufficient CAD harness through action
and observation sufficiency. This section follows an actual modeling decision
through \textsc{DeciCAD} to show how these conditions are instantiated. The
model first expresses geometric intent in a semantic action space, the harness
grounds that intent in the current CAD state, and the environment returns the
execution evidence needed for the next decision. Together, these mechanisms
determine the CAD policies that the model can learn and execute.

\subsection{Semantic CAD Action Space}

A modeling policy first requires an action vocabulary capable of expressing
design intent. \textsc{DeciCAD} represents each model output as a CAD action
with geometric arguments. The vocabulary covers documents and artifacts,
sketches and constraints, solid features, Booleans and patterns, edge
treatments, object placement, assembly, and state inspection. Each action
exposes the variables that determine the CAD state transition. An extrusion,
for example, specifies its profile, material mode, extent, and direction,
while object resolution and CAD API invocation remain within the harness. This
feature-oriented organization follows structured construction histories used
in parametric CAD corpora while exposing them as executable interaction steps
\citep{wu2021deepcad,willis2020fusion360gallery,khan2024text2cad}.

The following abstract interface summarizes the model-visible action space.
\texttt{Ref<T>} denotes a stable environment-generated object reference that
subsequent actions can reuse; the expression after \texttt{->} describes the
geometric postcondition returned after execution.

\begin{lstlisting}
CADAction ::=
    Document {Create, Open, Save, Export}
  | Sketch   {Create, AddGeometry, Constrain,
              Pattern, Trim, Project}
  | Solid    {CreateBody, Extrude, Revolve, Sweep,
              Loft, Boolean, Repeat, Fillet,
              Chamfer, Shell}
  | Entity   {SetPlacement}
  | Assembly {Create, AddComponent, Ground, AddMate,
              RemoveMate, Solve}
  | Inspect  {List, Get}

Extrude(
    profile:   Ref<Sketch>,
    mode:      Add | Cut,
    extent:    Distance(d) | ThroughAll,
    direction: SketchNormal | Vector3D
) -> {
    feature:     Ref<SolidFeature>,
    body:        Ref<Body>,
    solid_count: Int,
    volume_mm3:  Float
}
\end{lstlisting}

This representation identifies behaviorally equivalent choices of variable
names, control flow, and API syntax as the same action, while retaining the
parameters that determine geometric outcome. Stable references compose these
actions: a sketch can produce a solid feature; that solid can participate in
Booleans, transformations, and edge treatments; and parts can enter an
assembly and receive constraints. Each action describes a local geometric
transition, while an action sequence incrementally constructs the full design.
For the task distribution considered in this work, this compositional closure
allows long-horizon modeling policies to remain within one semantic action
space, instantiating action sufficiency in the main text.

\subsection{Grounding Semantic Actions in CAD State}

Semantic actions specify what the model can select; state grounding determines
what a selection means in the current design. A point, direction, or rotation
has a unique geometric meaning only relative to an explicit coordinate frame.
\textsc{DeciCAD} therefore interprets all action arguments under a shared
spatial contract. Linear quantities use millimeters and angular quantities use
degrees. Documents use a right-handed world frame with (+X) right, (+Y)
back, and (+Z) up, and rotations follow the right-hand rule. Two-dimensional
geometry lies in the right-handed local frame of its sketch, while solid
placement explicitly selects world or parent coordinates. System instructions
state these rules, typed arguments mark the space of each value, and the
execution layer reports authoritative frames read from the actual CAD
placement.

\begin{lstlisting}
SpatialContract {
    length_unit = millimeter
    angle_unit  = degree
    world_frame = right_handed(+X right, +Y back, +Z up)
    rotation    = right_hand_rule(axis, angle)

    Point2D   := coordinate in Ref<Sketch>.local_frame
    Point3D   := coordinate in world_frame
    Placement(space = World | Parent)
    Frame     := {origin, unit_x, unit_y, unit_z}
}

CreateSketch(support, plane)
    -> {sketch: Ref<Sketch>, local_frame: Frame}

SetPlacement(target, placement, space)
    -> {target, effective_placement, space}
\end{lstlisting}

The spatial contract fixes the meaning of arguments, while object references
fix their targets. Documents, entities, and reusable topological elements have
stable identifiers returned by creation or inspection and reused unchanged in
later actions. The harness resolves these references in the current document,
executes the native CAD operation, and verifies action-specific
postconditions. A valid result is committed to the persistent document so that
later actions operate on the resulting state; a failed result preserves the
previous state or marks the execution outcome as unknown.

\begin{lstlisting}
Execute(action, state s_t):
    args       <- validate(action, SpatialContract)
    objects    <- resolve(action.references, s_t)
    transaction.begin()
    native_out <- apply(action.name, args, objects)
    evidence   <- verify_postcondition(native_out)

    if evidence is valid:
        transaction.commit()
        return Success(evidence), s_{t+1}
    else:
        transaction.rollback()
        return Failure(recovery_evidence), s_t
\end{lstlisting}

Action grounding is therefore more than a lookup from a tool name to a backend
function. The shared spatial contract determines parameter semantics, stable
references determine the operated objects, and transactions with postcondition
checks determine whether an action actually took effect. Every successful
action in the model history corresponds to a verified transition in the
persistent CAD document. Together, these constraints implement the binding
map defined in the main text.

\subsection{Decision-Oriented Environment Observations}

After an action executes, the model must continue from the resulting
environment state. Expanding the complete CAD document repeatedly introduces
objects and properties irrelevant to the current decision, whereas reporting
only call-level success cannot support geometric judgment or recovery.
\textsc{DeciCAD} separates these needs: passive execution feedback describes
the state transition that just occurred, while active hierarchical inspection
retrieves local state needed by the next decision.

\subsubsection{Passive Execution Feedback}

The harness returns success or failure after every action. A success response
contains affected objects, action-specific postconditions, and the document
generation. A failure response contains an error category, the relevant
argument or object, and a recovery direction. Invalid input, missing
references, geometric failure, conflict, timeout, and unknown outcome remain
distinct because they call for different next actions.

\begin{lstlisting}
ActionResult :=
    Success {
        affected:       [StableRef],
        postcondition:  ActionSpecificEvidence,
        document_state: Generation
    }
  | Failure {
        category: InvalidInput | MissingObject |
                  InvalidState | GeometryFailure |
                  Conflict | Timeout |
                  BackendUnavailable | OutcomeUnknown,
        location: Optional<ArgumentPath | StableRef>,
        message:  ShortExplanation,
        recovery: RetryWithArguments | InspectState |
                  ChooseAnotherAction
    }
\end{lstlisting}

For example, a successful extrusion returns its generated feature, owning
body, and solid count; invalid arguments identify the responsible field; and
an unknown outcome directs the model to inspect the document before acting
again. These distinctions determine whether the policy should continue
modeling, revise an argument, or first verify state. Passive feedback therefore
retains causal evidence for local continuation and recovery while keeping each
step compact.

\subsubsection{Active Hierarchical Inspection}

Passive feedback focuses on the action just completed, but later decisions may
depend on a broader portion of the design. An edge treatment may require
locating topological edges with particular geometry, while an assembly mate
may require checking component placement and connectivity. For these cases,
\texttt{InspectList} first discovers relevant state. The model selects a
bounded collection---such as entities, children, topology, sketch geometry,
constraints, dependencies, or assembly interference---and receives stable
references with compact summaries through filtering and pagination.
\texttt{InspectGet} then reads the canonical detail of one precise reference,
including identity, placement or local frame, bounding box, shape or constraint
summary, relations, and validity.

\begin{lstlisting}
InspectList(
    collection: Documents | Entities | Children |
                Topology | SketchGeometry |
                SketchConstraints | Dependencies |
                Interferences | ArtifactComponents,
    scope:  Optional<Ref<Document | Entity> | ArtifactPath>,
    filter: Optional<{name_pattern, entity_type}>,
    offset: Int,
    limit:  Int
) -> {
    items: [{ref, kind, compact_summary}],
    generation: Generation,
    next_offset: Optional<Int>
}

InspectGet(
    target: Ref<Document | Entity | Face | Edge | Vertex>
            | ArtifactPath
) -> {
    identity, placement_or_frame, bounding_box,
    shape_or_constraint_summary, relation_summary, validity
}
\end{lstlisting}

The two inspection levels answer, respectively, which objects are relevant to
the current decision and what the authoritative state of one target is. A
policy can first list the topological edges of a solid, select candidates from
their summaries, retrieve exact geometric properties, and reuse the same
reference in a fillet action. Separating discovery from retrieval bounds each
observation while allowing the model to recover state distinctions that change
the next action. Passive feedback and active inspection thereby instantiate
observation sufficiency in the main text.

\subsection{Implementing the Semantic Decision Bottleneck}

Across this loop, \textsc{DeciCAD} coordinates compression of model-facing
actions and observations. Semantic actions merge behaviorally equivalent
implementation forms, while the spatial contract, stable references, and
execution verification ensure that the compressed actions still denote
well-defined CAD transitions. Passive feedback omits the complete internal
state, and active inspection restores local distinctions when a decision
requires them. Compression therefore remains constrained by action and
observation sufficiency.

Together, these mechanisms instantiate the Semantic Decision Bottleneck from
the main text. The model generates CAD decisions rather than their expanded
implementations and observes state evidence supporting the next decision
rather than the full internal representation of the CAD system. Coordinating
actions, grounding, and observations in the same closed loop reduces
representation burden without removing the decision paths required for task
completion and error recovery.

In the concrete code organization, these mechanisms are separated into three
connected components. The harness runtime organizes the decision loop across
model invocation, action execution, and feedback. An interface-contract layer
defines semantic actions, geometric representations, and execution results.
A backend execution layer binds those contracts to native operations in a
specific CAD system. This organization allows the interaction policy, the
model-visible interface, and CAD execution logic to evolve independently. A
new CAD backend implements the same action and feedback contracts while
preserving the model-facing decision space.

\section{Supplementary Experiments}

This section completes the evaluation protocol used in the main text and
further analyzes how the harness affects CAD capability, interaction
representation, and policy learning. The main text reports aggregate geometry
results, difficulty trends, output scaling, observation ablations, policy
performance, and learned behavior. Here we provide the protocols and evidence
not shown in the main text: difficulty calibration and per-quantile geometry,
the geometry--context trade-off, qualitative assets, supervision construction,
online-RL design, and a per-task trajectory example. All analyses retain the frozen tasks, models,
and evaluator from the main experiments.

\subsection{Complete Experimental Setup}

We first specify the common benchmarks, baselines, difficulty calibration,
geometry evaluator, and policy-learning protocol underlying the subsequent
analyses.

\paragraph{Benchmarks and baselines.}
The harness-level evaluation covers four CAD settings. CADPrompt uses its
complete evaluation set. We manually reviewed CADInstruct and P3D-Text-to-3D
and fixed their most challenging subsets before running any method.
P3D-Assembly-3D comprises multimodal assembly tasks that provide both textual
requirements and target images. Together, the benchmarks cover parametric
part modeling, complex text-conditioned geometry, and multi-component
assembly \citep{alrashedy2025cadcode,lv2025cadinstruct,khan2024text2cad,
    willis2020fusion360gallery}.

\begin{table}[h]
    \centering
    \footnotesize
    \setlength{\tabcolsep}{3pt}
    \begin{tabular}{@{}lccc@{}}
        \toprule
        Benchmark & Samples & Input & Task type \\
        \midrule
        CADPrompt & 200 & Text & Part modeling \\
        P3D-Text-to-3D & 100 & Text & Part modeling \\
        CADInstruct & 100 & Text & Part modeling \\
        P3D-Assembly-3D & 203 & Text + image & Assembly modeling \\
        \bottomrule
    \end{tabular}
    \caption{Statistics of the benchmarks used in our experiments.}
    \label{tab:benchmark-statistics}
\end{table}

Table~\ref{tab:benchmark-statistics} summarizes the input modality and task
type of each benchmark. Together they span text-conditioned part modeling and
multimodal assembly with progressively richer geometric and relational
requirements.

We compare zero-shot CadQuery, CADSmith, CAD-Assistant, Claude Code, and
\textsc{DeciCAD}. Zero-shot emits a complete CadQuery program in one response.
CADSmith coordinates planning, program generation, and verification across
multiple agents. CAD-Assistant iteratively revises a design through a Python
interpreter and external CAD operations. Claude Code provides a general coding
harness for CadQuery generation, execution, and repair. \textsc{DeciCAD}
instead exposes semantic CAD actions and object-level environment feedback
\citep{mallis2024cadassistant,barkley2026cadsmith,
    anthropic2026claudecode}.
Harness-level experiments use Claude Opus 4.7, identical task inputs and
geometry evaluation, and only the system instructions required by each
interface \citep{anthropic2026opus47}.

\begin{figure*}[t]
    \centering
    \begin{subfigure}[t]{0.45\textwidth}
        \includegraphics[width=\linewidth]{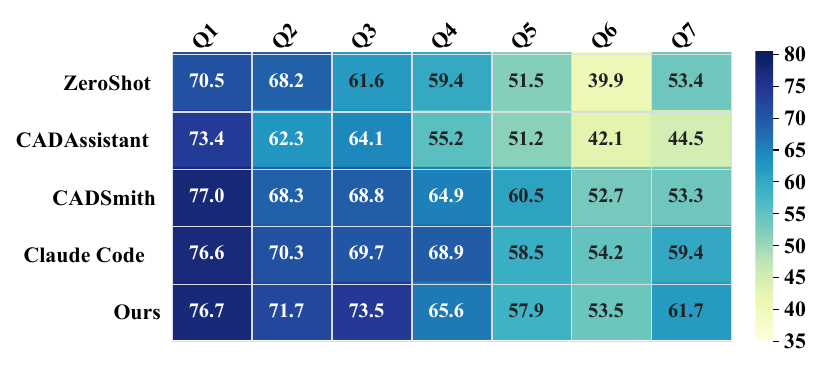}
        \caption{CADPrompt}
    \end{subfigure}\hfill
    \begin{subfigure}[t]{0.45\textwidth}
        \includegraphics[width=\linewidth]{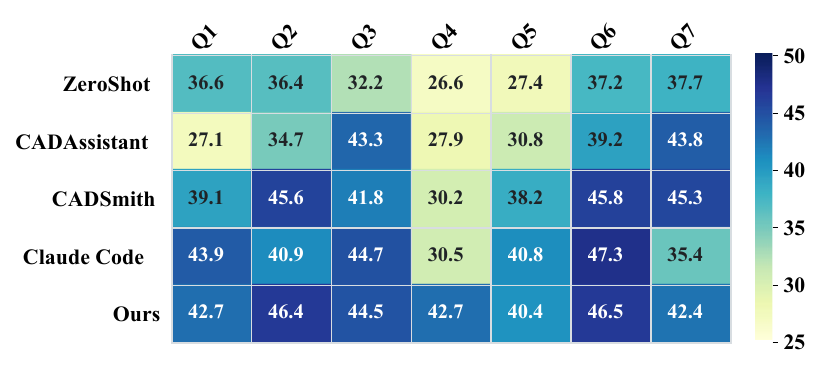}
        \caption{P3D-Text-to-3D}
    \end{subfigure}

    \begin{subfigure}[t]{0.45\textwidth}
        \includegraphics[width=\linewidth]{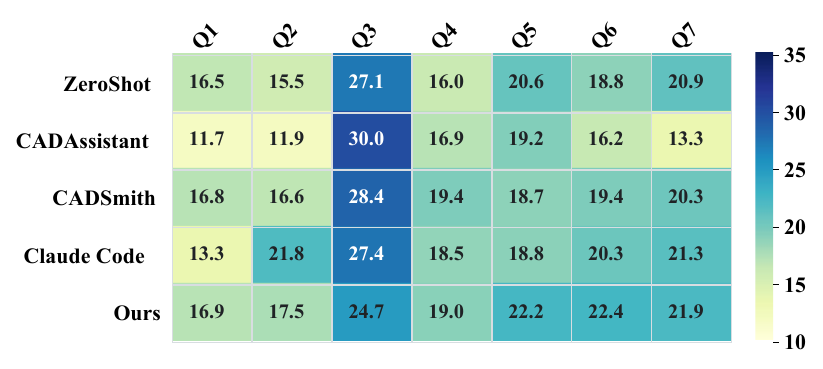}
        \caption{CADInstruct}
    \end{subfigure}\hfill
    \begin{subfigure}[t]{0.45\textwidth}
        \includegraphics[width=\linewidth]{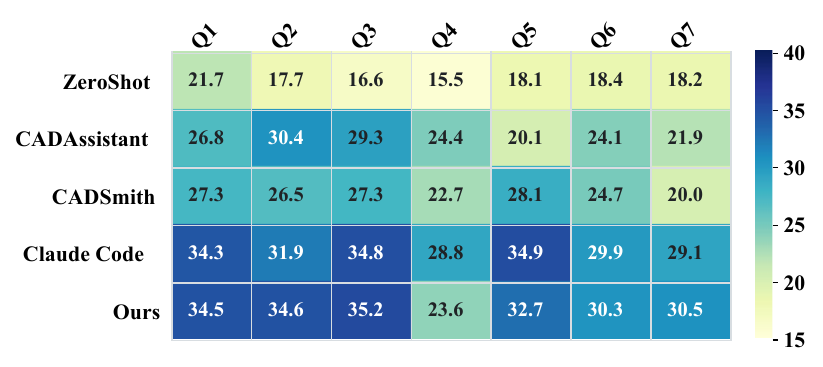}
        \caption{P3D-Assembly-3D}
    \end{subfigure}
    \caption{Mean end-to-end F1 across calibrated within-dataset difficulty quantiles.}
    \label{fig:supp-f1-quantiles}
\end{figure*}

\paragraph{Difficulty calibration.}
Aggregate means alone do not locate a method's boundary on complex tasks. For
task \(x\), let \(s_{\mathrm{sem}}(x)\) denote semantic complexity,
\(s_{\mathrm{mesh}}(x)\) reference-mesh complexity, and \(a(x)\) description
ambiguity. We first define geometric complexity and then combine it with
ambiguity:
\begin{align}
    g_{\mathrm{geo}}(x)
     & =\alpha s_{\mathrm{sem}}(x)+\beta s_{\mathrm{mesh}}(x), \\
    d(x)
     & =\lambda_1g_{\mathrm{geo}}(x)+\lambda_2a(x)
    +\lambda_3g_{\mathrm{geo}}(x)a(x).
\end{align}
The coefficients are fixed across benchmarks. The interaction term assigns
higher difficulty when complex geometry is also described ambiguously. We
compute \(d(x)\) by the same rule for every sample and form seven
equal-frequency quantiles \(Q1\)--\(Q7\) within each dataset. Quantiles thus
represent relative difficulty within a benchmark rather than identical
absolute complexity across benchmarks.

\begin{table}[h]
    \centering
    \footnotesize
    \begin{tabular}{@{}ll@{}}
        \toprule
        Component & Setting \\
        \midrule
        Student / teacher & Qwen3.5-9B / Claude Sonnet 4.6 \\
        SFT LoRA & rank 32, alpha 64, dropout 0.05 \\
        SFT learning rate & $1\times10^{-5}$, cosine \\
        SFT context & 65,536 tokens \\
        Checkpoint interval & 25 optimizer steps \\
        GRPO LoRA & rank 32, alpha 64 \\
        Rollout group & 16 tasks $\times$ 8 samples \\
        Rollout context & 131,072 tokens \\
        Sampled-sequence cap & 65,536 tokens \\
        Hardware & Eight NVIDIA H800 GPUs \\
        \bottomrule
    \end{tabular}
    \caption{Shared policy-learning configuration.}
    \label{tab:supp-policy-settings}
\end{table}

\paragraph{Geometry and execution evaluation.}
Generated and reference assets are converted to a common mesh representation,
scale-normalized, center-aligned, and rigidly aligned by ICP. Chamfer Distance
(CD) and surface F1 use \(10{,}000\) uniformly sampled surface points per
asset; the F1 nearest-neighbor threshold and IoU voxel resolution are both
\(0.01\). Geometry Availability is the fraction of tasks producing a nonempty,
parseable asset that enters geometry evaluation. F1, IoU, and CD are computed
over those valid assets.
For the difficulty-stratified
Figures~\ref{fig:supp-f1-quantiles} and
\ref{fig:supp-f1-slopes}, we additionally report end-to-end F1:
missing, unparsable, or invalid outputs are assigned F1=0.
Availability and conditional geometry quality are
kept separate rather than collapsed into a single success score.
The surface-distance and occupancy measures follow established point-set and
implicit-geometry evaluation practice
\citep{fan2017pointset,mescheder2019occupancy}.

Table~\ref{tab:supp-policy-settings} reports the shared policy-learning
configuration. The two harness tracks use the same student and teacher models,
LoRA capacity, context limits, checkpoint schedule, and hardware.

\paragraph{Policy-learning protocol.}
Policy learning uses the Text2CAD/DeepCAD 800/200 train--test split. Claude
Sonnet 4.6 generates teacher trajectories independently in the coding harness
and \textsc{DeciCAD}; Qwen3.5-9B is the shared student. Both tracks use the
same task descriptions, reference geometry, CAD environment, and evaluator.
The final test set is excluded from trajectory selection, balancing,
hyperparameter adjustment, reward design, and checkpoint selection.

SFT and online RL use BF16 LoRA. The SFT tracks share learning rate, batch
configuration, context length, and LoRA settings. GRPO begins from the selected
SFT checkpoint for each harness and samples student rollouts through the same
interface used during supervision.

\begin{figure*}[t]
    \centering
    \begin{subfigure}[t]{0.24\textwidth}
        \includegraphics[width=\linewidth]{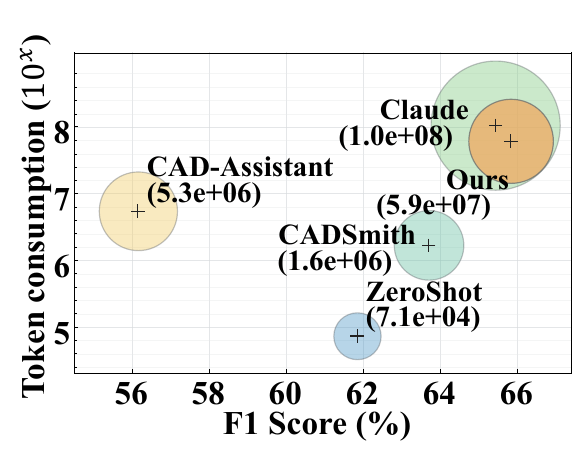}
        \caption{CADPrompt}
    \end{subfigure}\hfill
    \begin{subfigure}[t]{0.24\textwidth}
        \includegraphics[width=\linewidth]{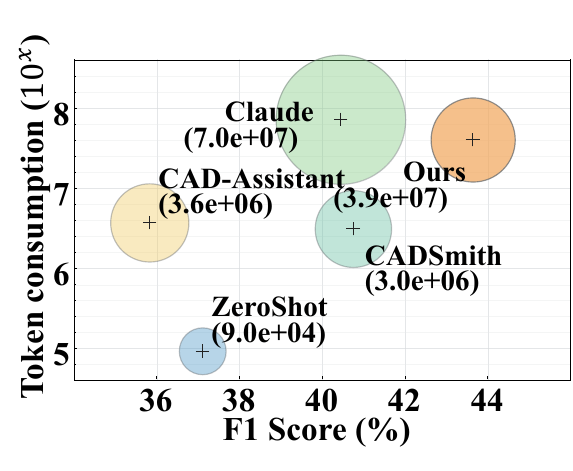}
        \caption{P3D-Text-to-3D}
    \end{subfigure}\hfill
    \begin{subfigure}[t]{0.24\textwidth}
        \includegraphics[width=\linewidth]{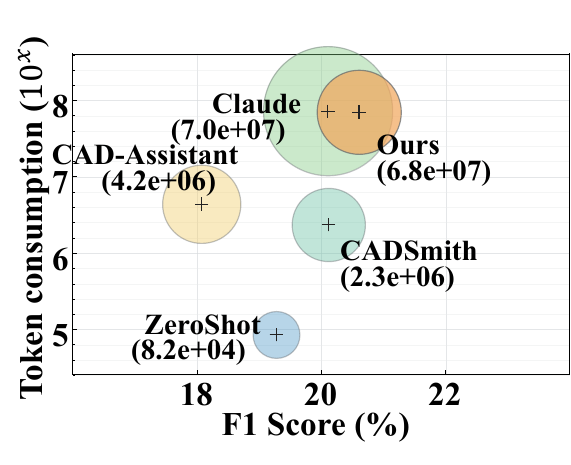}
        \caption{CADInstruct}
    \end{subfigure}\hfill
    \begin{subfigure}[t]{0.24\textwidth}
        \includegraphics[width=\linewidth]{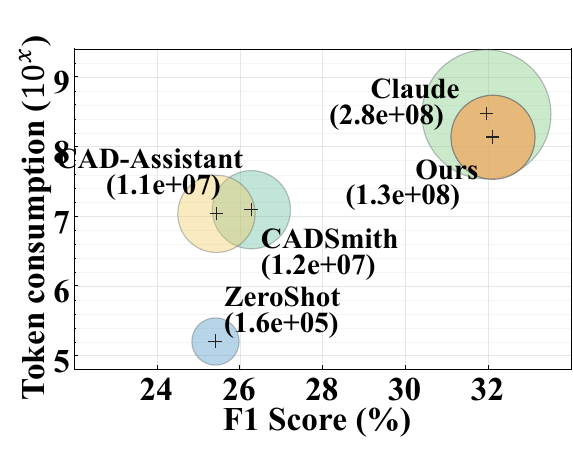}
        \caption{P3D-Assembly-3D}
    \end{subfigure}
    \caption{Geometry--context trade-off across benchmarks. Each point reports
        mean F1 and cumulative model-visible input tokens for one method.}
    \label{fig:supp-f1-input}
\end{figure*}

\subsection{Extended Harness Comparison}

We extend the aggregate comparison in the main text along task difficulty,
model-visible context, and representative output geometry.

\begin{figure}[h]
    \centering
    \includegraphics[width=\linewidth]{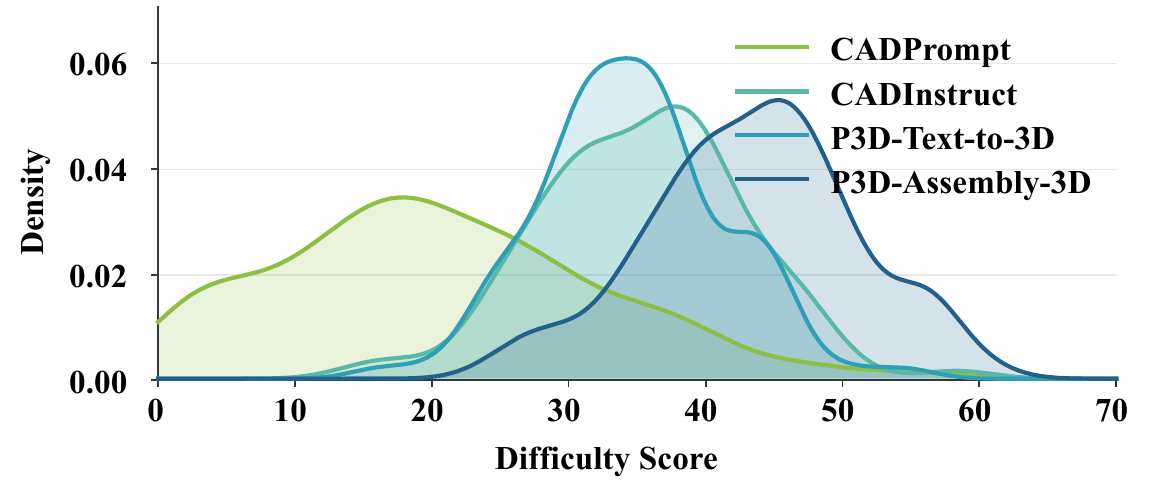}
    \caption{Calibrated difficulty distributions of the four benchmarks.}
    \label{fig:supp-difficulty-distribution}
\end{figure}

\paragraph{Difficulty distribution and stratified performance.}
Figure~\ref{fig:supp-difficulty-distribution} shows the calibrated difficulty
distributions. CADPrompt concentrates in the lower range, CADInstruct and
P3D-Text-to-3D shift toward greater complexity, and P3D-Assembly-3D occupies
the highest range. This ordering reflects the additional need to preserve
component identity, relative placement, and cross-object relations in
assembly.

Figure~\ref{fig:supp-f1-quantiles} expands the aggregate F1 results across
\(Q1\)--\(Q7\). \textsc{DeciCAD} leads 14 of the 28 dataset--quantile cells,
more than any baseline, and exceeds Claude Code at \(Q7\) on every benchmark.
It ranks first in five of seven P3D-Assembly-3D quantiles and leads from \(Q5\)
through \(Q7\) on CADInstruct. On P3D-Text-to-3D, its F1 remains ahead of
Claude Code at both \(Q4\) (42.7 versus 30.5) and \(Q7\) (42.4 versus 35.4).

The aggregate advantage is therefore not driven only by easy tasks. If the
semantic action space omitted operations required by complex designs,
performance would deteriorate systematically at high difficulty. Its
competitive \(Q7\) behavior instead indicates that compressing
implementation-level choices retains the necessary CAD transitions.
CADInstruct remains non-monotonic, so difficulty serves as a stratification
variable rather than a deterministic predictor of failure.

Figure~\ref{fig:supp-f1-slopes} in the main text summarizes these quantile
trends with linear fits. Because CADInstruct is non-monotonic, its slope is
used only as a compact summary of the heatmap.

\begin{figure*}[t]
    \centering
    \includegraphics[width=0.75\textwidth]{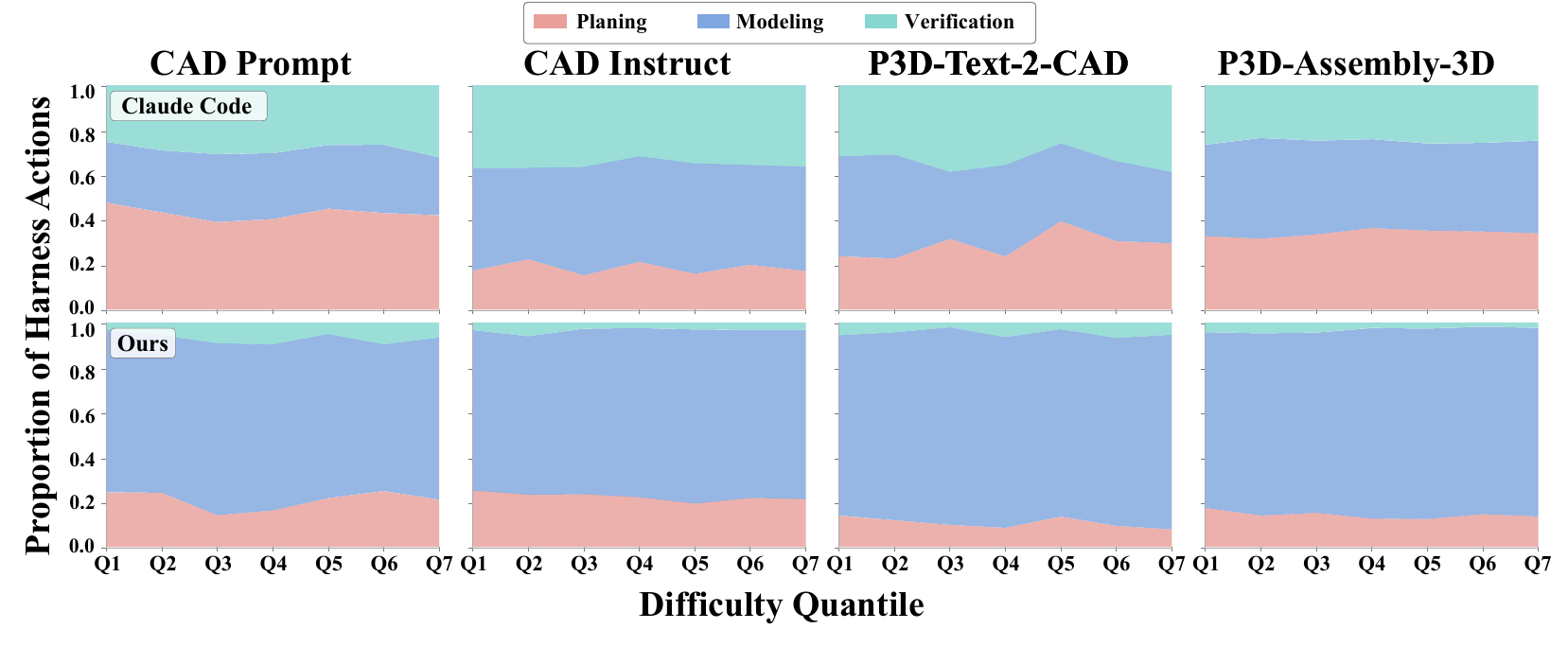}
    \caption{Allocation of observable decisions to Planning, Modeling, and
        Verification across datasets and difficulty quantiles. The upper and
        lower rows show Claude Code and \textsc{DeciCAD}, respectively.}
    \label{fig:supp-decision-allocation}
\end{figure*}

\begin{figure*}[t]
    \centering
    \includegraphics[width=0.485\textwidth]{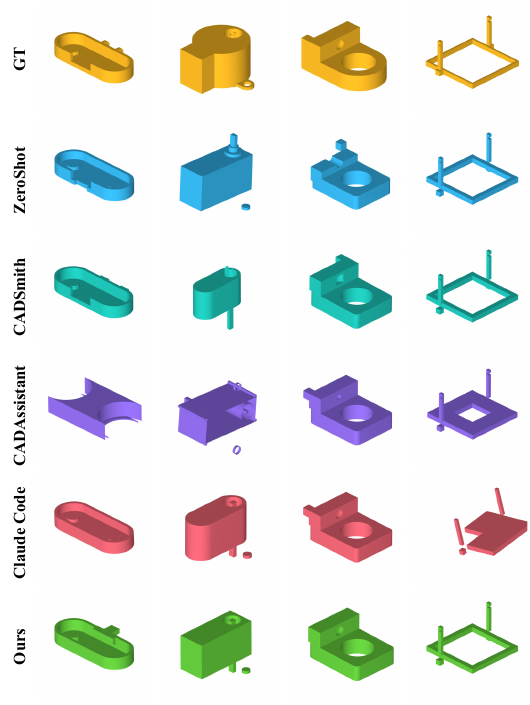}
    \hfill
    \includegraphics[width=0.485\textwidth]{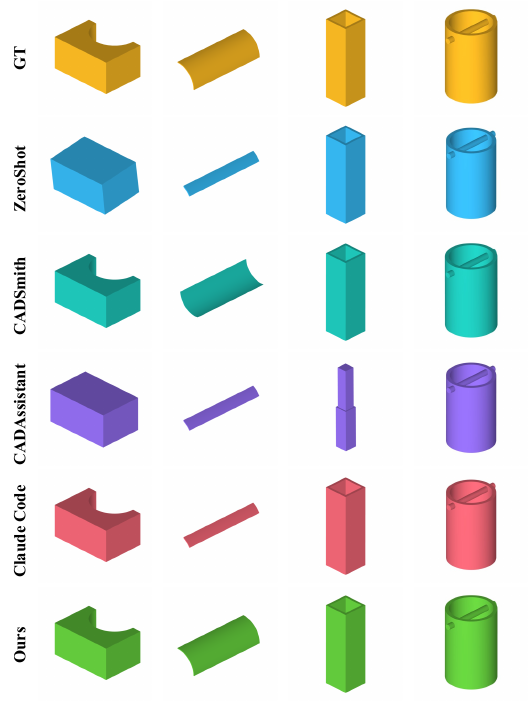}
    \caption{Qualitative comparison on P3D-Text-to-3D (left) and CADPrompt
        (right).}
    \label{fig:supp-qualitative-text}
\end{figure*}

\paragraph{Geometry--context trade-off.}
Figure~\ref{fig:supp-f1-input} jointly compares geometric quality and
model-visible input tokens. On the three text-to-CAD benchmarks,
\textsc{DeciCAD} occupies a stronger performance--cost region than Claude
Code: it attains higher F1 while consuming substantially less input context.
The geometric advantage therefore cannot be attributed to exposing the model
to a larger interaction history. P3D-Assembly-3D is more demanding for both
harnesses because multi-object references and repeated state updates enlarge
the context, yet \textsc{DeciCAD} retains a competitive trade-off. Together
with the output-side analysis below, these results separate representation
efficiency from trajectory truncation or reduced task coverage.

\paragraph{Output scaling with interaction length.}
Figure~\ref{fig:supp-turns-output} in the main text compares interaction turns
with cumulative output at the task level and shows that the output reduction
does not arise from early termination.

\paragraph{Decision allocation.}
We group decisions as Planning, Modeling, or Verification according to their
external effect. Figure~\ref{fig:supp-decision-allocation} shows that
\textsc{DeciCAD} consistently devotes a larger share of observable decisions
to Modeling across datasets and difficulty levels, whereas Claude Code devotes
a larger combined share to Planning and Verification.

\paragraph{Qualitative comparison.}
Figure~\ref{fig:supp-qualitative-text} shows representative assets from
P3D-Text-to-3D and CADPrompt. Zero-shot and lightweight specialized agents
more often omit local features or produce unevaluable artifacts. Claude Code
usually recovers the global shape but may retain geometric-detail errors
despite executable code. \textsc{DeciCAD} more consistently preserves the
requested profiles, openings, and repeated features, matching the aggregate
trends in F1, IoU, and Availability.

\subsection{Harness-Grounded SFT}

We next describe how closed-loop teacher trajectories are converted into
causal supervision and how the resulting policies are selected and evaluated.

\paragraph{Causal decision supervision and data composition.}
Policy supervision must preserve the decision conditions under which each
action was produced. For every teacher action \(a_t\), we reconstruct the
model-visible history \(h_t\) immediately before that action, including the
system instruction, task, preceding model actions, tool results, and available
tool schema. The target is the teacher's next coding response or semantic CAD
action. When a trajectory contains a local failure, the failure result remains
in the history and supervision is applied to the subsequent recovery decision.
This causal construction follows the principle that interactive policies must
be trained on histories induced by preceding actions rather than on future
trajectory information \citep{ross2011reduction,qin2023toolllm}.

Both datasets originate from the official Text2CAD training split. Claude
Sonnet 4.6 executes the same tasks independently through the Coding and
\textsc{DeciCAD} harnesses, and the resulting assets are evaluated against the
same target geometry. Geometry- and protocol-audited trajectories are then
expanded into causal decision examples
\citep{khan2024text2cad,anthropic2026sonnet46,qwen2026qwen35}. The
\textsc{DeciCAD} dataset contains
642 audited tasks, divided by task identity into 542 training and 100
development tasks. Each task retains at most 11 regular and three corrective
decisions covering initialization, core modeling, inspection, recovery, and
artifact persistence. The Coding dataset contains 499 geometry-valid tasks,
of which 471 are used for training and 28 for development, yielding 3,651
training targets and 207 development targets; 336 training targets follow an
observed execution error.

\begin{table}[h]
    \centering
    \footnotesize
    \setlength{\tabcolsep}{3pt}
    \begin{tabular}{@{}lccc>{\raggedright\arraybackslash}p{1.8cm}@{}}
        \toprule
        Harness & Audited tasks & Train / dev. & Target cap & Recovery targets \\
        \midrule
        Coding & 499 & 471 / 28 & 12 & Up to 4 observed \\
        \textsc{DeciCAD} & 642 & 542 / 100 & 11 + 3 & Up to 3 observed or controlled \\
        \bottomrule
    \end{tabular}
    \caption{SFT supervision obtained through the two harnesses. Both tracks
        use the same student model and optimization protocol.}
    \label{tab:supp-sft-data}
\end{table}

The two tracks share the source tasks, teacher and student models, causal
truncation, context limit, and optimization protocol, while retaining the
natural granularity of their respective decision spaces. \textsc{DeciCAD}
examples are stratified by task difficulty, CAD stage, accumulated document
state, preceding action and outcome, history pattern, and target action class.
Generic query and computation actions are capped, whereas geometry-changing,
inspection, and recovery actions retain their causal histories. Coding data
follow the same task-level filtering principle and prioritize genuine recovery responses.

Table~\ref{tab:supp-sft-data} records the audited task coverage and target
construction used for the two SFT tracks. These counts reflect the executable
supervision available in each harness-induced decision space.

\paragraph{Checkpoint selection and learned policies.}
Both tracks save a checkpoint every 25 optimizer steps. Training and
token-level validation loss monitor optimization, while the final checkpoint
is selected on frozen development tasks using closed-loop Availability, F1,
IoU, and CD.

SFT equips both Qwen3.5-9B policies with closed-loop CAD capability, but the
learned outcomes differ. Coding SFT raises Availability from \(12.5\%\) to
\(60.5\%\), F1 from \(0.5902\) to \(0.6460\), and IoU from \(0.5871\) to
\(0.6286\), while CD changes from \(0.00972\) to \(0.01130\). The student
therefore acquires program construction, execution, and repair, although local
geometric error remains after the large coverage expansion. \textsc{DeciCAD}
SFT raises Availability from \(61.5\%\) to \(66.5\%\), F1 from \(0.6883\) to
\(0.7437\), and IoU from \(0.6487\) to \(0.6631\), while reducing CD from
\(0.01023\) to \(0.00782\). Semantic CAD supervision thus improves both asset
coverage and the geometry of valid outputs. The aggregate stage-wise geometry
results and the resulting policy behavior are reported in the main text; below
we detail the online optimization that produces them.

\subsection{Geometry-Grounded Online RL}

Starting from the selected SFT policies, we detail the curriculum, reward, and
update constraints used for online optimization before analyzing the learned
behavior.

\paragraph{Two-stage online optimization.}
Online RL continues from states induced by the current student. Training tasks
serve four roles: \texttt{completion\_frontier} tasks frequently fail to form a
valid asset; \texttt{recovery} tasks expose locally recoverable failures;
\texttt{quality\_frontier} tasks complete but admit geometric improvement; and
\texttt{anchor} tasks are solved reliably and constrain forgetting. The first
stage emphasizes completion and recovery to expand valid-asset coverage. Its
policy becomes the frozen parent of a second, geometry-oriented stage. The
final second-stage iteration samples 16 tasks---12 quality-frontier, two
completion-frontier, and two anchor tasks---with eight rollouts per task.
This ordering combines curriculum learning with group-relative online
optimization used in tool-learning policies
\citep{bengio2009curriculum,shao2024deepseekmath,qian2025toolrl,
    gong-etal-2026-toolcad}.

Each group contains three conservative, three standard, and two exploratory
rollouts with \((T,p)=(0.65,0.90)\), \((0.90,0.95)\), and
\((1.10,0.98)\), respectively. Interaction budgets increase with task stage
from 32, 40, and 48 to 64 decisions. Both stages retain the same harness, tool
contracts, and terminal evaluator; the second stage changes policy preference
within the existing CAD decision space.

\paragraph{Geometry and process reward.}
For a rollout with valid terminal geometry, the normalized geometry reward is
\begin{equation}
    R_{\mathrm{geo}}
    =w_F F1^{p_F}+w_I IoU^{p_I}
    +w_C\exp\!\left(-\frac{CD}{\tau_C}\right),
\end{equation}
where \(w_F=0.45\), \(w_I=0.35\), \(w_C=0.20\), \(p_F=p_I=2\), and
\(\tau_C=0.002\). Squaring retains discrimination in the high-quality region,
and the exponential maps unbounded CD to \([0,1]\). Missing or unparsable
assets, and assets with \(R_{\mathrm{geo}}<0.15\), receive zero geometry reward.

Process evidence describes observable CAD progress. At decision \(t\), let \(v_t\),
\(c_t\), \(r_t\), and \(b_t\) indicate a valid call, a new state change,
recovery from the preceding failure, and repeated failure or no-progress
regression. Then
\begin{equation}
    R_{\mathrm{proc},t}=\operatorname{clip}
    \left(0.20v_t+0.55c_t+0.25r_t-0.45b_t,-1,1\right),
\end{equation}
and trajectory process reward is its decision-wise mean. Protocol reward is
the valid-call fraction. The second-stage reward is
\begin{equation}
    R=
    \begin{cases}
        0,                       & R_{\mathrm{geo}}=0, \\
        \eta R_{\mathrm{geo}}+\rho R_{\mathrm{proc}}
        +\xi R_{\mathrm{proto}}, & R_{\mathrm{geo}}>0,
    \end{cases}
\end{equation}
with \((\eta,\rho,\xi)=(0.94,0.05,0.01)\). Zero-geometry gating places the
terminal asset first, while the smaller process and protocol terms distinguish
no-progress behavior, invalid calls, and executable modeling sequences.

\begin{figure*}[t]
    \centering
    \includegraphics[width=0.96\textwidth]{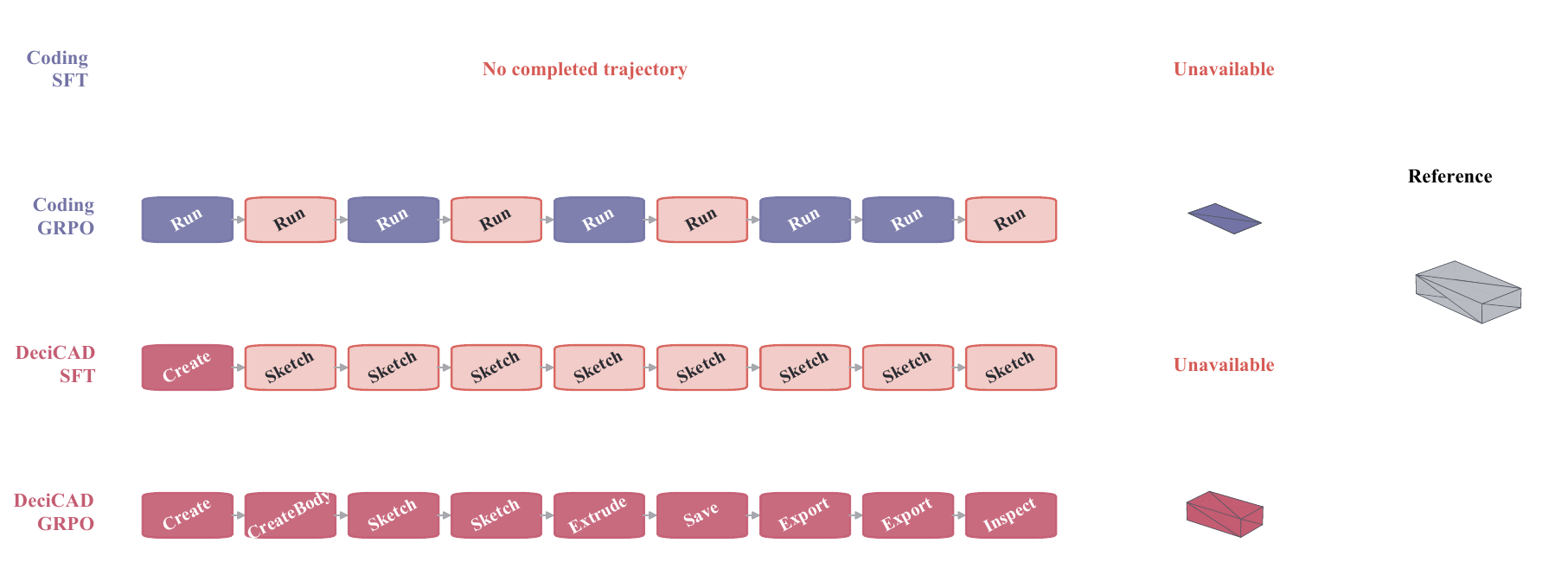}
    \caption{Cross-harness policy trajectories and terminal geometry on the
        same CAD task. Rows show Coding SFT, Coding GRPO, \textsc{DeciCAD} SFT, and
        \textsc{DeciCAD} GRPO. Nodes summarize external calls and their execution
        status; the right side shows generated assets and the reference geometry.
        Unavailable indicates that no evaluable asset was produced.}
    \label{fig:supp-policy-trajectory}
\end{figure*}

\paragraph{Group-relative advantage and update gating.}
Eight rollouts of one task form a GRPO group. A group enters optimization only
if at least one rollout has nonzero geometry reward and its reward standard
deviation is at least \(0.02\). This excludes all-failure groups whose
trajectories differ only in procedural score.

The second stage combines completion and quality advantage at \(0.15:0.85\).
Quality-frontier and anchor candidates must satisfy an absolute quality floor
and remain within the tolerance of the frozen parent anchor. Each group retains
at most one positive and one negative trajectory, balancing task influence.
The policy objective uses learning rate \(3\times10^{-9}\), two-sided clip
\(0.03\), and token-level KL coefficient \(\beta=0.10\) to the frozen
first-stage parent. A stop action receives positive credit only when geometry
exceeds both an absolute threshold and a high within-group quantile; premature
termination receives negative credit.

The learned second-stage LoRA increment is finally applied to the parent at
strength \(0.375\). We preserve the adapter weights and change LoRA alpha from
64 to 24, scaling the learned offset by \(24/64=0.375\). This continuous
control retains the parent's coverage and recovery behavior while adding the
quality preference learned in the second stage.

\paragraph{Online optimization results.}
\textsc{DeciCAD} GRPO raises Availability from \(66.5\%\) to \(73.5\%\), F1
from \(0.7437\) to \(0.7880\), and IoU from \(0.6631\) to \(0.7036\), while
reducing CD from \(0.00782\) to \(0.00570\). Coding GRPO raises Availability
from \(60.5\%\) to \(65.5\%\), but F1 falls from \(0.6460\) to \(0.6253\),
IoU from \(0.6286\) to \(0.5987\), and CD increases from \(0.01130\) to
\(0.02253\). Thus, identical terminal geometry feedback expands executable
coverage in both harnesses, while stable geometry improvement depends on how
credit propagates through the harness-induced decision space.

\paragraph{Observable capabilities in rollouts.}
Figure~\ref{fig:supp-policy-behavior} in the main text reports policy behavior
under six shared categories. Construct creates or modifies geometric
implementation; Inspect executes code, reads state, or checks a result;
Recover changes action after an explicit failure and restores progress;
Invalid Reference and No Progress record invalid targets and repeated
non-progress; and Persist saves or exports the terminal asset. Events are
counted per 100 external decisions so that trajectories of different lengths
are comparable.

Figure~\ref{fig:supp-policy-trajectory} grounds the aggregate difference in
one shared task. Coding SFT produces no completed trajectory. Coding GRPO
emits an asset after repeated runs and repairs, but the output collapses to a
thin plate unlike the reference structure. \textsc{DeciCAD} SFT enters the
domain modeling process but repeated failed Sketch actions leave no evaluable
asset. \textsc{DeciCAD} GRPO instead completes a chain of creation, sketching,
extrusion, persistence, export, and inspection, producing terminal geometry
consistent with the reference. The example clarifies that recovery count alone
does not imply useful recovery: subsequent actions must continue to advance
the actual CAD state.

Aggregate geometry, behavior statistics, and trajectory-level evidence jointly
define the scope of online improvement. GRPO increases valid-asset coverage and
reduces invalid references and repeated non-progress, while small features in
long-horizon tasks remain limited by terminal reward resolution. F1 and IoU
measure surface and volumetric agreement, and CD adds local distance, but
feature count, topological correspondence, repeated structures, and constraint
relations can still differ under similar terminal scores. We therefore use
closed-loop geometry as the primary evidence and the observed behavior and
assets to show how improvement forms within each harness decision space.

\begin{figure*}[t]
    \centering
    \begin{subfigure}[t]{0.485\textwidth}
        \includegraphics[width=\linewidth]{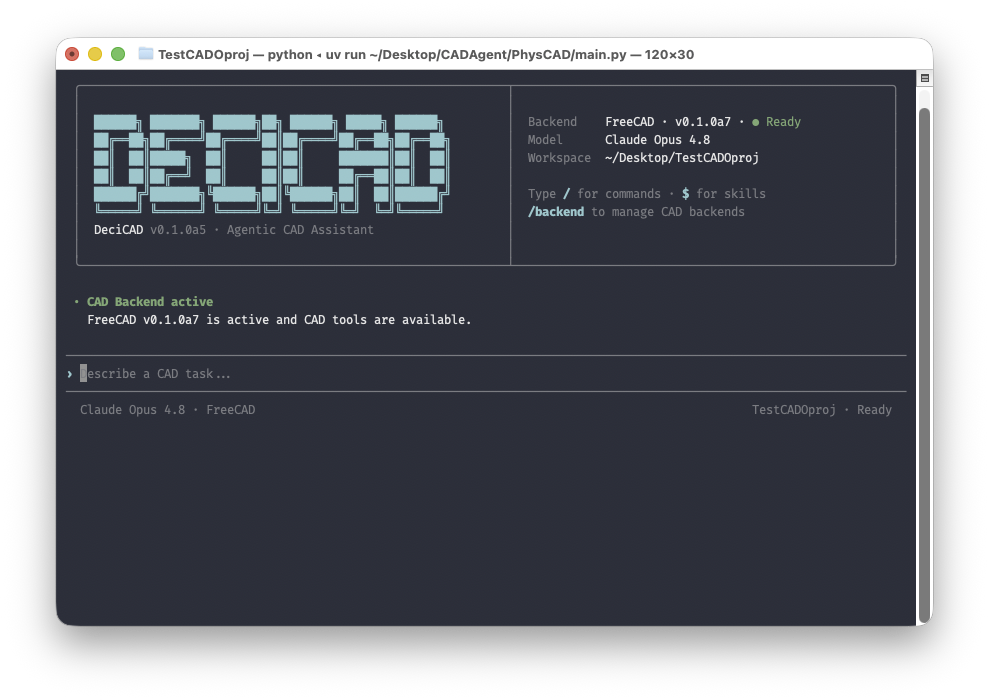}
        \caption{Project workspace and active CAD runtime.}
    \end{subfigure}\hfill
    \begin{subfigure}[t]{0.485\textwidth}
        \includegraphics[width=\linewidth]{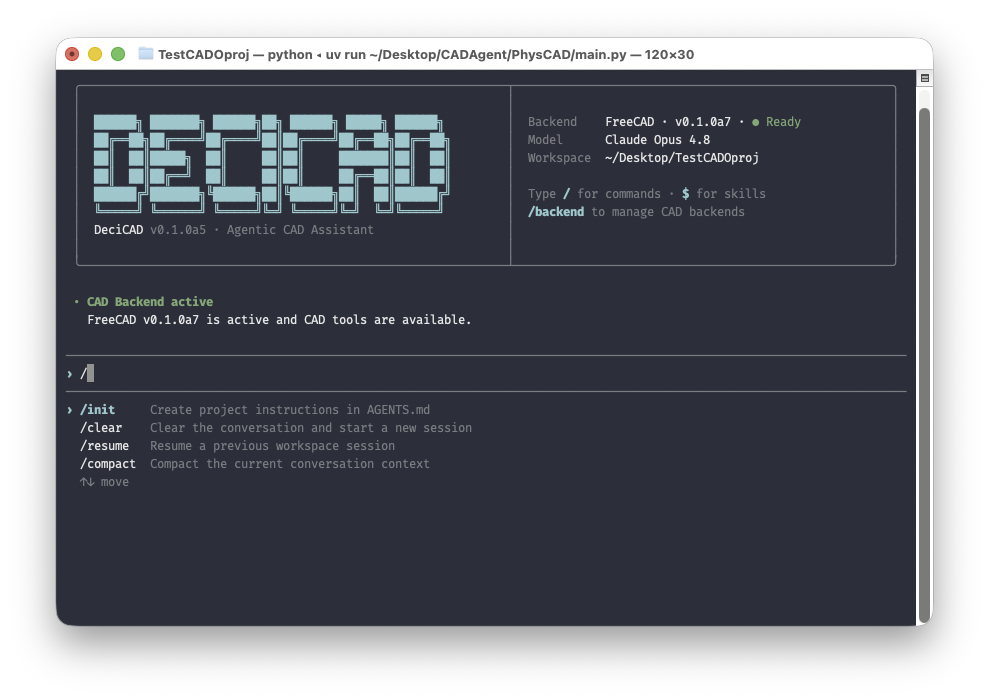}
        \caption{Session commands available from the composer.}
    \end{subfigure}

    \vspace{4pt}
    \begin{subfigure}[t]{0.485\textwidth}
        \includegraphics[width=\linewidth]{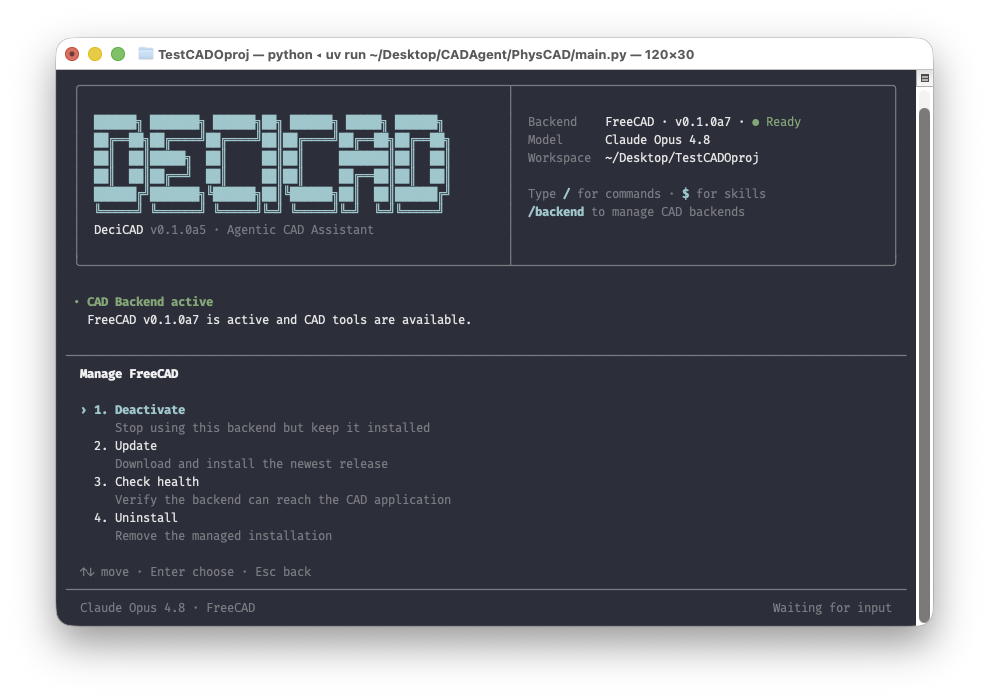}
        \caption{Lifecycle and health controls for the FreeCAD backend.}
    \end{subfigure}\hfill
    \begin{subfigure}[t]{0.485\textwidth}
        \includegraphics[width=\linewidth]{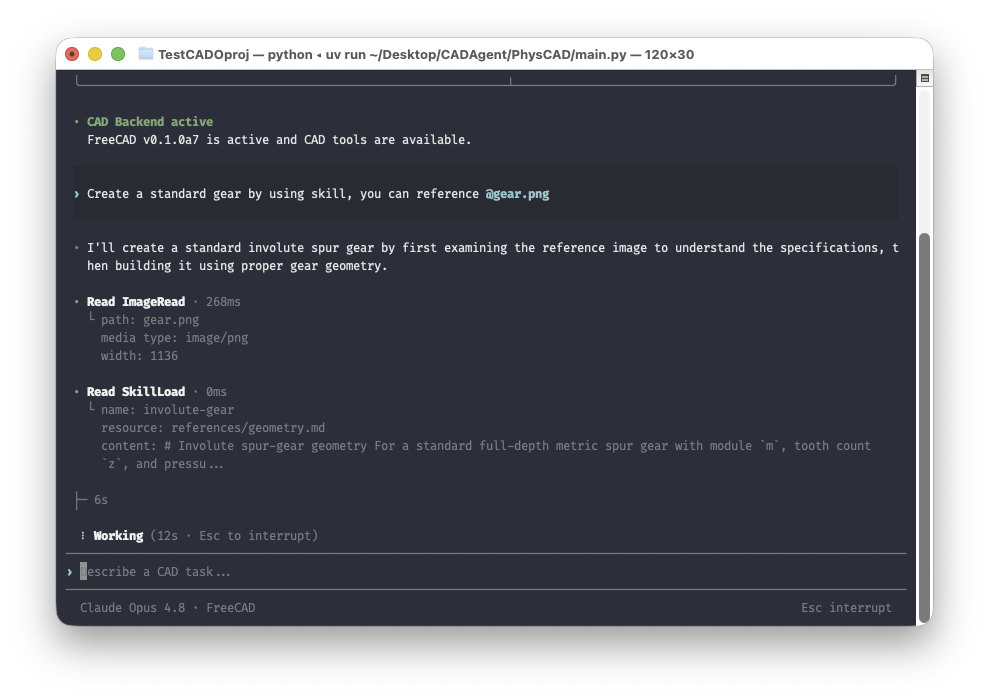}
        \caption{A live task using an image reference and a CAD skill.}
    \end{subfigure}
    \caption{The project-level \textsc{DeciCAD} interface. The workspace keeps
        natural-language interaction, project resources, CAD skills, backend
        state, and streamed execution events in one persistent session. Users
        can inspect and control the runtime while semantic CAD actions continue
        to operate on the active document.}
    \label{fig:supp-decicad-interface}
\end{figure*}

\section{\textsc{DeciCAD}: An Interactive Agentic CAD System}

\textsc{DeciCAD} is distributed as a terminal-native Agentic CAD application.
It organizes model providers, specialized agents, CAD backends, and project
resources in one workspace, allowing a user to create, modify, inspect,
assemble, and export native CAD documents through natural-language
interaction. This section presents the resulting product workflow, its support
for long-horizon multi-agent projects, the backend SDK available to developers,
and our official FreeCAD integration. Its streamed action--observation loop
follows the interaction pattern established by language agents while grounding
each domain action in persistent CAD state
\citep{yao2022react,yang2024sweagent}.

\begin{figure*}[t]
    \centering
    \includegraphics[width=0.94\textwidth]{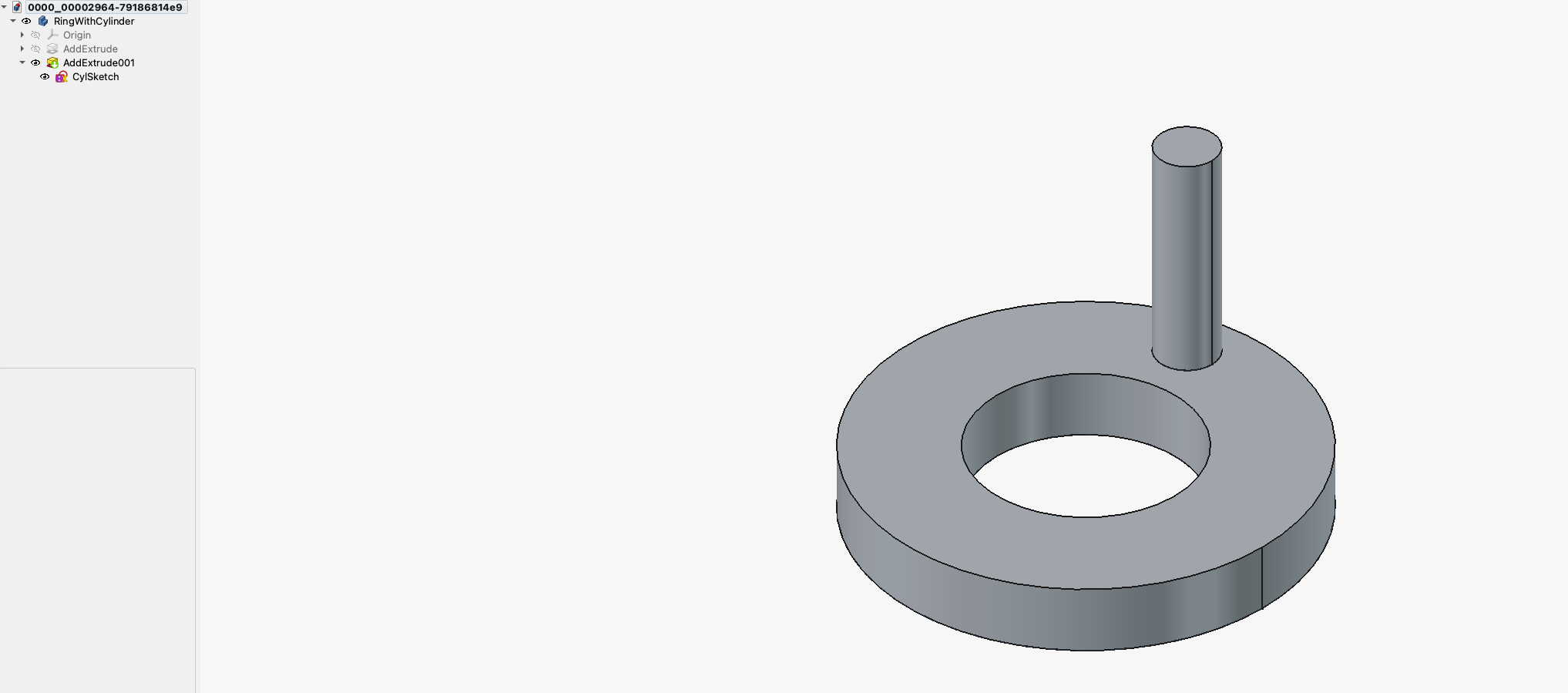}
    \caption{An editable CAD document produced through the FreeCAD backend.
        The viewport shows the generated solid, while the document tree retains
        the body, additive features, and underlying sketch created during the
        interaction. This persistent structure allows later inspection and
        modification through either \textsc{DeciCAD} or the native FreeCAD
        interface.}
    \label{fig:supp-freecad-output}
\end{figure*}

\subsection{From Natural Language to Editable CAD Documents}

\textsc{DeciCAD} provides native release packages for macOS, Linux, and Windows
and uses a managed installer for application updates. Running \texttt{decicad}
inside a project opens that directory as the active workspace;
\texttt{decicad --project /path/to/project} selects it explicitly. During
initial setup, the application configures a model provider and requests trust
before accessing project files or CAD resources.

The composer accepts part modeling, assembly, geometric revision, inspection,
and export requests. An \texttt{@} reference attaches a project resource such
as an image or specification, while \texttt{\$} loads an inline CAD skill that
supplies task-specific procedures. Slash commands expose runtime controls in
the same interface: \texttt{/model} selects the provider, model, and reasoning
effort; \texttt{/backend} installs, activates, updates, and checks a CAD
backend; and \texttt{/status} summarizes the current runtime and token usage.
Figure~\ref{fig:supp-decicad-interface} shows these operations in the
terminal interface.

Model actions and environment results then appear in temporal order. File and
image reads identify the evidence used by the model, skill loads expose the
selected domain procedure, and CAD calls report their execution state and
compact result. Successful calls immediately update the persistent document;
failed calls return information that supports argument correction, targeted
inspection, or selection of another action. The user can interrupt execution
or add a follow-up request at any point, and subsequent actions continue from
the current CAD state rather than regenerating the artifact from scratch.

The resulting artifact is also available in the native CAD application. Users
can inspect the three-dimensional view and access the created Body, Sketch,
Feature, Assembly, and Component objects in the document tree. The document can
be saved in its native editable format or exported as a standard geometry
asset, preserving a path from agent generation to manual inspection and
continued engineering work.

\subsection{Control Plane for Multi-Agent CAD Work}

The interactive loop is supported by an application-level control plane that
coordinates agents, tools, resources, permissions, and persistent project
state. CAD backends execute domain operations, but they do not decide which
agent may act, which calls may run concurrently, or whether an action was based
on a current document revision. Keeping these responsibilities in the control
plane allows the same orchestration policy to operate across compatible CAD
systems.

\paragraph{Agent orchestration.}
For larger projects, a leader agent retains the user conversation and overall
goal while delegating bounded subtasks to four built-in roles. The Designer
turns functional and geometric intent into a feature and interface plan; the
Modeler owns one complete part or a dependency-ordered modeling stage; the
Assembler inserts components and manages placements, grounding, mates, and
solve state; and the Verifier performs an independent, predominantly read-only
evidence pass over the final document and exported assets. Each child receives
the original request together with a bounded assignment, separate context,
role-specific instructions, tool allowlist, and interaction budget. A child
cannot recursively create more agents or modify the leader's task board. Its
final response is returned as evidence together with the resources changed
during execution, leaving acceptance and follow-up decisions with the leader.

\paragraph{Task and dependency state.}
Multi-stage work is represented by a compact project task graph. Tasks record
their status, description, dependencies, blocking relations, and completion
evidence. This graph separates conversational planning from execution state:
the leader can recover which components are ready, which assembly step is
blocked, and which verification remains after context compaction or session
resumption. Dependency updates are serialized by the control plane, preventing
two agents from concurrently producing inconsistent task-board transitions.

\paragraph{Resource resolution and authorization.}
Every tool specification declares static resource behavior, including the
arguments that identify reads, writes, observations, creations, or filesystem
materialization. Before execution, the resource manager resolves those fields
to concrete document, entity, topology, assembly, or file keys. Each agent also
maintains the revisions it has actually observed. A write is authorized only
when its dependent resources have been observed at their current revisions;
if another agent has modified one of them, the stale write is rejected with a
recoverable conflict and the agent must inspect the updated state. File edits
follow the same rule and require a complete prior read. This observation-based
authorization prevents a plausible but outdated model context from silently
overwriting newer CAD or project state.

\paragraph{Concurrency and commit.}
Independent child assignments can be emitted together and run concurrently.
Tool specifications additionally assign calls to execution lanes and mark them
as parallel or exclusive. After resource resolution, write scopes are sorted
and locked in a deterministic order; calls with disjoint scopes proceed in
parallel, while overlapping CAD writes are serialized. Authorization is
revalidated immediately before execution to close the interval between
scheduling and dispatch. A successful backend result reports concrete
\texttt{RuntimeResourceEffects}; the control plane commits these effects,
advances resource generations, records the responsible agent, and invalidates
stale observations held elsewhere. A failed call does not claim a state
transition. If transport fails after a potentially mutating call, the affected
scope is conservatively invalidated so that later decisions re-inspect it.

This protocol supports useful concurrency without treating the CAD document as
an ordinary shared text buffer. Different Modelers may construct independent
parts simultaneously, after which an Assembler consumes their committed
outputs. Calls that touch the same body, assembly, or output file are ordered,
and each downstream agent receives only the document revisions known to be
valid at delegation time.

\paragraph{Permissions and project boundaries.}
Workspace trust and permission checks are evaluated before protected files or
CAD resources are exposed to an agent. The user remains the authority for
access that crosses the active project boundary or carries elevated effects.
Because permissions are evaluated in the application rather than reimplemented
by every backend, adding a CAD integration does not create a second,
backend-specific authorization path.

\paragraph{Sessions, memory, and observability.}
The control plane projects model messages, tool calls, execution outcomes,
permissions, resource commits, task changes, and compaction events onto an
append-only rollout. The same event stream drives the terminal interface and
provides the state required for replay and session restoration.
\texttt{/resume} reconstructs an earlier project session, while
\texttt{/compact} replaces older interaction with a bounded summary and keeps
recent decisions available to the model. Durable memory is extracted from
completed root rollouts with links to its source turns, separating persistent
design facts from transient dialogue. Skills add reusable CAD procedures,
hooks connect project workflows to agent and tool lifecycle events, and a
project-level \texttt{AGENTS.md} carries coordinate conventions, naming rules,
and delivery requirements into subsequent sessions. Together, these services
allow long-running CAD work to remain recoverable, inspectable, and governed by
one consistent project state.

\begin{table*}[t]
    \centering
    \scriptsize
    \setlength{\tabcolsep}{5pt}
    \renewcommand{\arraystretch}{1.12}
    \begin{tabular}{p{0.105\textwidth} p{0.36\textwidth} p{0.47\textwidth}}
        \toprule
        Domain & Canonical operations                                                 & Supported behavior \\
        \midrule
        Document
               & \texttt{DocumentCreate}, \texttt{DocumentOpen},
        \texttt{DocumentSave}, \texttt{DocumentExport},
        \texttt{DocumentClose}, \texttt{DocumentDeleteEntity}
               & Create and open native FreeCAD documents; save editable documents;
        export selected targets to STEP or STL; close documents; and delete
        unused entities without downstream dependents.                                                     \\

        Inspection
               & \texttt{InspectList}, \texttt{InspectGet}
               & Page through bounded document, entity, geometry, topology, and
        assembly collections, then retrieve canonical detail for one stable
        reference.                                                                                         \\

        Entity
               & \texttt{EntitySetPlacement}
               & Apply rigid transforms in world or parent space to placeable CAD
        entities and component instances.                                                                  \\

        Sketch
               & \texttt{SketchCreate}, \texttt{SketchAddLines},
        \texttt{SketchAddCircle}, \texttt{SketchAddArc},
        \texttt{SketchAddPoint}, \texttt{SketchConstrain},
        \texttt{SketchProject}, \texttt{SketchPattern},
        \texttt{SketchTrim}
               & Create plane-bound sketches; add local lines, circles, arcs, and
        points; apply geometric and dimensional constraints; project
        external geometry; create linear or circular patterns; and trim at
        intersections.                                                                                     \\

        Solid
               & \texttt{SolidCreateBody}, \texttt{SolidExtrude},
        \texttt{SolidRevolve}, \texttt{SolidSweep}, \texttt{SolidLoft},
        \texttt{SolidBoolean}, \texttt{SolidRepeat},
        \texttt{SolidFillet}, \texttt{SolidChamfer}, \texttt{SolidShell}
               & Create feature bodies; add or cut material by extrusion, revolution,
        sweep, or loft; perform body booleans; repeat or mirror features and
        bodies; and apply fillets, chamfers, and inward or outward shells.                                 \\

        Assembly
               & \texttt{AssemblyCreate}, \texttt{AssemblyAddComponent},
        \texttt{AssemblySetGrounded}, \texttt{AssemblyMateAdd},
        \texttt{AssemblyMateRemove}, \texttt{AssemblySolve}
               & Create assemblies; insert placed component instances; ground or
        release components; add and remove geometry-referenced mates; and
        solve the resulting assembly constraints.                                                          \\
        \bottomrule
    \end{tabular}
    \caption{Canonical CAD operations implemented by the current FreeCAD
        backend. The 34 tools constitute the complete SDK backend registry used
        by \textsc{DeciCAD}.}
    \label{tab:supp-freecad-tools}
\end{table*}

\subsection{Adapting a CAD Backend with the DeciCAD SDK}

The DeciCAD SDK is the contract boundary shared by the application and backend
implementations. It defines canonical tool names, typed input and output
models, geometry types, opaque object references, permission requirements,
runtime resource effects, structured success values, and model-facing errors.
The application remains responsible for model interaction, agent orchestration,
permissions, sessions, and resource coordination. A backend developer maps the
canonical state transitions to the native API of a target CAD system.

This separation makes the SDK a semantic contract rather than a convenience
wrapper around one vendor API. Canonical tools are grouped into document,
inspection, entity, sketch, solid, and assembly domains. Their Pydantic models
normalize units, frames, geometric values, operation modes, and result shapes,
while \texttt{DocumentRef} and \texttt{EntityRef} preserve backend-owned object
identity as opaque values. A backend imports these specifications instead of
copying or redefining them. Consequently, field descriptions, model-facing tool
schemas, resource declarations, and output validation remain identical across
integrations, while implementation instructions can still describe behavior
specific to the target CAD application.

A backend first imports a canonical \texttt{ToolSpec} and binds it to an async
implementation with \texttt{@tool}. Inputs have already passed the spec's
Pydantic validation before execution. The implementation returns the canonical
output model together with accurate \texttt{RuntimeResourceEffects}; these
effects let the application update resource versions, permissions, and locks.
Recoverable CAD failures use a compact \texttt{ToolError}, keeping operational
failures distinct from backend defects and transport errors. The following
example summarizes the binding used for native document creation.

\begin{lstlisting}[language=Python,caption={Minimal binding of a canonical SDK tool to a native CAD operation.},label={lst:supp-backend-tool}]
CREATE = DOCUMENT_TOOLS["DocumentCreate"]

@tool(spec=CREATE,
      instructions="Create one native document.")
async def create_document(arguments):
    result = await execute_cad_write(
        operation="new_document",
        arguments=arguments,
    )
    if result.ok:
        record_document_create_effect(result)
    return result
\end{lstlisting}

The backend then declares a manifest containing its backend, SDK, protocol, and
tool-contract versions. A \texttt{BackendServer} registers the complete
canonical tool roster and exposes a health check for the target CAD runtime, as
illustrated in Listing~\ref{lst:supp-backend-server}. During activation, a
\texttt{BackendClient} starts the isolated process and verifies the manifest,
protocol version, tool registry, contract digest, and health response. Only
after this initialization does it bind remote executors as ordinary SDK tools
inside \textsc{DeciCAD}.

\begin{lstlisting}[language=Python,caption={Registering tool domains and exposing backend health through the SDK server.},label={lst:supp-backend-server}]
server = BackendServer(
    manifest=MANIFEST,
    health_check=health_check,
)
for domain in (document, entity, inspection,
               sketch, solid, assembly):
    server.register(domain.TOOLS)
await server.serve()
\end{lstlisting}

The server reserves standard output for newline-delimited JSON-RPC frames and
sends operational diagnostics to standard error, keeping transport parseable
under subprocess supervision. It refuses to initialize when a canonical tool
is missing, a tool name is duplicated, an imported spec differs from the SDK,
or the manifest targets incompatible SDK, protocol, or tool-contract versions.
Successful results are validated against the exact output model before they
cross the process boundary; operational failure results contain a
\texttt{ToolError} and cannot claim resource effects.

These constraints reduce the work required from a backend developer. The SDK
provides model schemas, validation, RPC framing, process lifecycle, contract
negotiation, error serialization, and the client-side \texttt{Tool} bindings.
The developer concentrates on object resolution, native CAD calls, readiness,
and accurate resource effects. A real \texttt{BackendClient}/
\texttt{BackendServer} initialization test exercises the same registry and
wire contract used by the product, catching integration drift before an agent
can invoke the backend. The application and backend may then be released
independently whenever their declared contracts remain compatible; an
incompatible release fails explicitly at activation instead of producing
ambiguous behavior midway through a modeling task.

\subsection{FreeCAD as the Reference Backend}

Our official reference backend instantiates this SDK boundary for FreeCAD,
and all experiments in this paper use this backend. The
backend manager shown in Figure~\ref{fig:supp-decicad-interface} controls its
installation, activation, health checking, update, and removal through an
isolated managed environment. In development mode, the setup command also
installs a FreeCAD Addon into the versioned user directory discovered from the
actual FreeCAD runtime, supporting system, portable, and custom installations.

The integration uses two local communication layers. The SDK
\texttt{BackendClient} communicates with an independent backend process; that
process connects through an authenticated loopback bridge to FreeCAD's
embedded Python runtime. The bridge executes CAD operations on the FreeCAD GUI
thread, wraps document writes in native transactions, and encodes the result
under the canonical output model. The bridge panel exposes connection state,
dynamic endpoint, protocol version, age of the last authenticated request, and
start, stop, restart, and automatic-start controls, while the application sees
only the stable SDK contract. Keeping the bridge protocol separate is necessary
because FreeCAD owns the embedded interpreter and GUI event loop; it also keeps
FreeCAD object lookup, argument decoding, result encoding, transaction handling,
and native diagnostics out of the application process.

Table~\ref{tab:supp-freecad-tools} summarizes the complete canonical operation
set currently implemented by the reference backend. Stable opaque references
connect documents, entities, geometry, topology, features, and assembly objects
across calls. Successful tools report concrete resource effects, and failures
are mapped into shared error categories, so the general agent and resource
runtime do not depend on the FreeCAD Python API.

Inspection is hierarchical rather than a full document dump. Paginated list
operations expose document structure and stable references, after which get
operations retrieve detailed properties, placement, geometry, topology, or
assembly state only for selected objects. Modeling tools return compact native
results and resource effects immediately after execution. This supports the
same active and passive observation pattern used by the DeciCAD decision loop
while retaining direct access to FreeCAD's editable feature history.

Figure~\ref{fig:supp-freecad-output} shows a document state produced through
this path. The three-dimensional solid is accompanied by its body, additive
features, and sketch in the FreeCAD tree, rather than being exposed only as a
flattened mesh or image.

For batch evaluation and data collection, multiple workers run isolated
FreeCAD and backend processes with separate bridge endpoints, authentication
tokens, instance identifiers, and result caches. A worker resets and verifies
all open documents before accepting another episode. This isolation prevents
CAD state from leaking across tasks while retaining the same backend contract
used by the interactive product.

Parallelism is therefore process-isolated rather than implemented by issuing
concurrent CAD operations inside one FreeCAD instance. Each worker owns one
FreeCAD process, one Addon bridge, and one matching backend process. Distinct
endpoint descriptor files lead to independent dynamic loopback ports and
authentication tokens. This design supports parallel interactive services, 
trajectory collection, and benchmark evaluation without sharing mutable 
FreeCAD state across tasks.

\begin{table}[h]
    \centering
    \footnotesize
    \setlength{\tabcolsep}{3pt}
    \begin{tabular}{@{}lccc@{}}
        \toprule
        Capability & FreeCAD & SolidWorks & AutoCAD \\
        \midrule
        Documents and objects & \checkmark & \checkmark & \checkmark \\
        Inspection and query & \checkmark & \checkmark & \checkmark \\
        Parametric modeling & \checkmark & \checkmark & \\
        Assemblies and mates & \checkmark & \checkmark & \\
        2D drafting & & & \checkmark \\
        Exchange formats & STEP, STL & STEP, STL & DXF \\
        \bottomrule
    \end{tabular}
    \caption{Capabilities supported by the three CAD backends. All backends
        implement the same SDK contract; all experiments in this paper use the
        FreeCAD backend.}
    \label{tab:supp-backend-capabilities}
\end{table}

\subsection{Additional Backends}

The same SDK contract also supports two commercial CAD systems, SolidWorks and
AutoCAD. Each backend declares the capabilities it implements, and the
application loads the corresponding tools when it connects; agent
orchestration, permissions, resource coordination, and sessions are shared
across all backends. Table~\ref{tab:supp-backend-capabilities} summarizes the
capabilities supported by the three backends.

\end{document}